\documentclass[letterpaper]{article} 
\usepackage{aaai2026}  
\usepackage{times}  
\usepackage{helvet}  
\usepackage{courier}  
\usepackage[hyphens]{url}  
\usepackage{graphicx} 
\usepackage{natbib}  
\usepackage{caption} 
\usepackage{algorithm}
\usepackage{algorithmic}
\usepackage{amsfonts}       
\usepackage{amsmath,amssymb,amsthm}
\usepackage{caption}
\usepackage{microtype}      
\newtheorem{theorem}{Theorem}
\usepackage{bm}
\RequirePackage{graphicx}
\usepackage{booktabs}
\usepackage{longtable}
\usepackage{newfloat}
\usepackage{listings}
\DeclareCaptionStyle{ruled}{labelfont=normalfont,labelsep=colon,strut=off} 
\floatstyle{ruled}
\newfloat{listing}{tb}{lst}{}
\floatname{listing}{Listing}
 \makeatletter
 \def\set@curr@file#1{\def\@curr@file{#1}} 
 \makeatother
 \usepackage[load-configurations=version-1]{siunitx} 
 
 \newcommand{\indicator}[1]{\mathbb{I}}

\title{Imputation of Unknown Missingness in Sparse Electronic Health Records}
\author{
Jun Han\textsuperscript{\rm *}\!\!\!\!\!\! 
  \quad Josue Nassar\textsuperscript{\rm *} \!\!\!\!\!\!   
  \quad Sanjit Singh Batra\textsuperscript{\rm *}\!\!\!\!\!\! \\
  \quad Aldo Cordova-Palomera\!\!\!\!
  \quad Vijay Nori\!\!\!\! 
  \quad Robert E. Tillman \\ 
}
\affiliations{
    \textsuperscript{\rm *} Equal Contribution\\
    Optum AI \\
    \{jun\_han, josue\_nassar, sanjit.batra, aldocp, vijay.nori, rob.tillman\}@optum.com
}

\usepackage{bibentry}

\newcommand{\noisyx}{\mathbf{\tilde{x}}}
\newcommand{\nan}{\textrm{NaN}}
\newcommand{\vx}{\mathbf{x}}

\begin{document}
\maketitle
\begin{abstract}
Machine learning holds great promise for advancing the field of medicine, with electronic health records (EHRs) serving as a primary data source. 
However, EHRs are often sparse and contain missing data due to various challenges and limitations in data collection and sharing between healthcare providers. 
Existing techniques for imputing missing values predominantly focus on \emph{known unknowns}, such as missing or unavailable values of lab test results; most do not explicitly address situations where it is difficult to distinguish what is missing. 
For instance, a missing diagnosis code in an EHR could signify either that the patient has not been diagnosed with the condition or that a diagnosis was made, but not shared by a provider. 
Such situations fall into the paradigm of \emph{unknown unknowns}.
To address this challenge, we develop a general purpose algorithm for denoising data to recover unknown missing values in binary EHRs.
We design a transformer-based denoising neural network where the output is thresholded adaptively to recover values in cases where we predict data are missing. 
Our results demonstrate improved accuracy in denoising medical codes within a real EHR dataset compared to existing imputation approaches and leads to increased performance on downstream tasks using the denoised data. In particular, when applying our method to a real world application, predicting hospital readmission from EHRs, our method achieves statistically significant improvement over all existing baselines.
\end{abstract}


\section{Introduction}
Previous work has demonstrated the potential of applying machine learning to the medical domain for tasks such as 
disease prediction, clinical trial design, and health economics and outcome research~\citep{deng2021deep,liu2021machine,padula2022machine}.
A standard data source for ML in the medical domain is electronic health records (EHRs).
Additionally, many existing disease prediction models primarily utilize tabular formats, 
often transforming longitudinal EHR data into binary or categorical forms, 
rather than employing time-series forecasting methods~\citep{tabular_1, tabular_2, tabular_3, tabular_4}.

In practice however, EHRs, upon being binarized, are prone to having missing values due to a variety of reasons such as recording issues~\citep{goldstein2017opportunities} and under-reporting~\citep{chen2023missing}.
Moreover, the lack of data sharing between providers and the heterogeneity of EHR coding across institutions~\citep{hur2022unifying} amplifies this problem as EHR records for the same patient can vary across providers and institutions.
Training on EHR data in the presence of missingness can negatively impact model performance at test time.
Therefore, to ensure the efficacy of trained models, it is beneficial to pre-process the data by replacing missing values.

The \textit{de facto} approach for replacing missing values in data is via imputation~\citep{rubin1978multiple, gelmanbda04, mazumder2010spectral}, where missing values are replaced with imputed values computed using available data. 
Generally, imputation techniques are designed for the \emph{known} unknown scenario where it is known which fields are missing~(Figure~\ref{fig:example_figure}A).
For instance, a lab test was ordered by the physician, but the value was not recorded or made available in the EHR.
In practice, missing fields are commonly encoded with a special symbol to indicate that the value is missing, e.g.,~$\nan$~\citep{li2021imputation} or $-1$ for binary data that is otherwise $0$ or $1$.

In contrast, clinical codes in EHRs encoded as tabular data are typically represented as $1$ in the presence of a diagnosis code and represented as $0$ in the absence of a diagnosis code.
This binary representation can be ambiguous as a $0$ corresponding to a specific diagnosis code may indicate that the patient does not have the diagnosis \textbf{or} they do, but it is missing.
For example, suppose that the ICD code corresponding to diabetes is represented as $0$ in a tabular encoding of a patient's EHR.
Without additional supporting evidence, this could mean either: 1) that the patient \textbf{has not} been diagnosed with diabetes, i.e., a negative diagnosis or 2) that the patient \textbf{was} diagnosed with diabetes by a physician, but this diagnosis was not recorded by the physician or made available in the EHR.
Thus, unlike the known unknowns paradigm, it is unclear when data are missing. We therefore refer to this scenario as \textit{unknown unknowns}~(Figure~\ref{fig:example_figure}B).
This ambiguity makes the problem of imputation significantly more challenging, as it necessitates determining whether an absent code should be treated as missing or not, significantly increasing the complexity of the solution space~(Figure~\ref{fig:example_figure}B).
Moreover, recent works have shown that the missingness of negative diagnosis is a much more common issue in EHRs compared to the missingess of positive diagnosis codes~\citep{wu2023collecting,chen2023missing}.
While one could treat all codes of value $0$ as missing and apply traditional imputation approaches, the sparsity of EHRs prevents such a strategy from being a viable solution.

In this work, we propose an approach, Denoise2Impute, based on denoising, that allows for better recovery of unknown unknowns compared to SOTA imputation-based approaches.
Denoise2Impute, is a denoising neural network parameterized by a SetTransformer~\citep{lee2019set} that allows the model to learn relationships between different ICD codes, e.g., a positive diabetes diagnosis is correlated with high glucose.
While Denoise2Impute performs well, to increase its ability to distinguish whether a $0$ corresponds to a negative diagnosis or a missing diagnosis code, we systematically formulate this challenge of unknown unknowns in EHRs as a denoising problem and theoretically derive the optimal denoising function.
Inspired by the theory, we introduce an extension of Denoise2Impute, which we denote as Denoise2Impute-T, where data-dependent thresholds are computed to help stabilize the output of the network.
To assess the efficacy of our proposed approach, we compute the accuracy of the resulting denoising as well as perform comprehensive evaluations on a variety of downstream tasks using the denoised data. 
We demonstrate that our proposed approach outperforms baselines by $\sim2$\% AUPRC in denoising clinical codes within a large EHR dataset and the denoised EHRs outperform baseline approaches on a variety of clinical prediction tasks by $\sim0.5$\%.
\begin{figure*}[h!]
    \centering
    \includegraphics[width=0.7\textwidth]{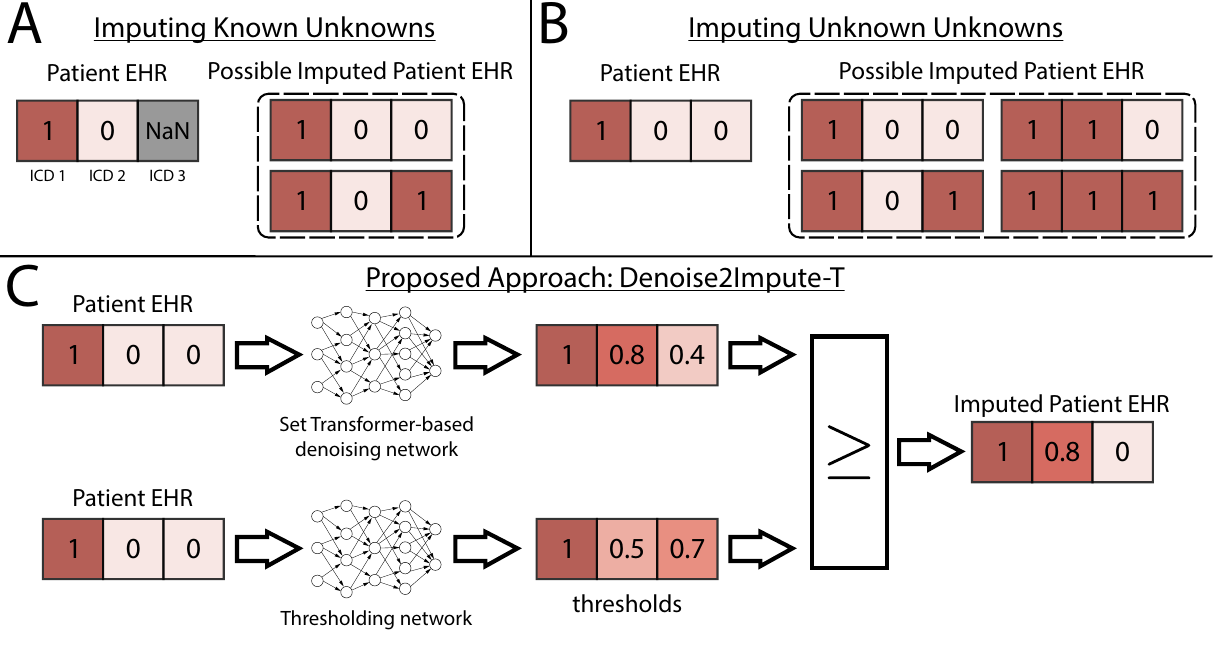}
    \caption{
        \textbf{A)} An example of imputing known unknowns.
        In the patient EHR, there is $\nan$ in ICD 3.
        Thus, there are only two possible choices for ICD 3: 0 or 1.
        \textbf{B)} An example of imputing unknown unknowns.
        In the patient EHR, there is a 0 in ICD 2 and ICD 3.
        Due to reporting issues, it is unclear whether a 0 represents a negative diagnosis (0) or a missing diagnosis (1).
        Due to this ambiguity, there are 4 possibilities after imputing: i) ICD 2 and 3 are both negative diagnoses, thus they both are 0.
        ii) ICD 2 was missing, thus it is 1. ICD 3 was a negative diagnosis, thus it is a 0.
        iii) ICD 2 was a negative diagnosis, thus it is a 0. ICD 3 was missing, thus it is 1.
        iv) Both ICD 2 and 3 are missing, thus they both are 1.
        \textbf{C)} Our proposed approach, Denoise2Impute-T, comprises 2 components.
        A patient EHR is input into a Set Transformer based denoiser that outputs a denoised patient EHR.
        The EHR is also input into a Thresholding network that outputs thresholds for each ICD code.
        A \textit{greater than or equal to} element-wise comparison is performed between the denoised EHR and the thresholds, that leads to the final imputed EHR.
        }
        \label{fig:example_figure}
\end{figure*}


\section{Related Work}
\paragraph*{Imputation and Denoising}
The process of denoising data sources has been extensively investigated with numerous domain-specific methods proposed \citep{brown1997introduction, buades2005non, mazumder2010spectral, stekhoven2012missforest, Jarrett2022HyperImpute, you2020handling, peis2022missing}, particularly in the medical domain \citep{chowdhury2017imputation, liu2023handling}. 
When it is known which entries are corrupted or missing, the denoising problem reduces to imputation, for which various techniques have been developed. 
In particular, K-nearest neighbor methods \citep{beretta2016nearest, bai2022modified, faisal2021imputation}, deep generative modeling approaches like variational autoencoders \citep{beaulieu2017missing, ma2018eddi}, and generative adversarial networks \citep{yoon2018gain, luo2018multivariate}. 
Techniques for imputation of time series medical data have been developed using diffusion models \citep{alcaraz2022diffusion, lee2022restoration, zheng2022diffusion}, and for the imputation of medical records using graph-based methods \citep{vivar2020simultaneous, vinas2021graph}. 
These methods have been further refined by incorporating attention-based approaches, which leverage learned correlations among known features to impute missing data \citep{lewis2021accurate, yildiz2022multivariate, wang2022multivariate, festag2023medical}.
In~\citep{chen2023missing}, an imputation algorithm based on interpretable machine learning was devised to allow for performant and interpretable imputation algorithms.

\paragraph*{Unknown Unknowns}
The problem of unknown unknowns has been studied in various domains, 
where it has also been referred to as latent missingess~\citep{wu2023collecting}, under-reported variables~\citep{sechidis2017dealing} and data with positive and unlabeled data~\citep{lazkecka2021estimating}.
In~\citet{wu2023collecting}, an active learning procedure for unknown unknowns in EHRs was proposed that takes into account the cost of checking if an absent diagnosis code is missing or negative.
\citet{sechidis2017dealing} and \citet{lazkecka2021estimating}
both focus on the classification setting where the label and/or the features can have unknown unknowns.
While similar in spirit to our proposed approach, these approaches are not scalable as they rely on computing conditional probabilities which is infeasible for EHRs as they can have $\sim$1,000 diagnosis codes.
In~\citep{suleiman2020clinical}, domain knowledge from expert clinicians is used to devise a recommender-system approach for imputing unknown unknowns in EHR data; we consider this work complimentary to ours.
In~\citet{wu2016collaborative}, a recommender-based system built on denoising autoencoders was to handle the cases of unknown unknowns in user preferences; we consider this work complimentary to ours.
Another domain that deals with unknown unknowns is genomics,
One class of approaches utilizes multiple noisy instances of a DNA sequence to build a denoised consensus sequence~\citep{edgar2015error, wenger2019accurate}. 
Another class of denoising approaches for DNA sequences utilizes additional covariates to perform denoising~\citep{wang2020nanoreviser}.
In this paper, however, we focus on denoising when there is a single noisy data point where it is not known \textit{a priori} which fields of the data point are corrupted or missing.
\section{Denoise2Impute: Denoising Unknown Unknowns in EHRs}
\label{sec:imputehr}
{\bf Background and Notation} We begin by introducing notation and reviewing the goal of denoising.
Let $\vx$ denote noiseless data, where $\vx \sim p(\vx), \vx \in \mathcal{X}$. 
Let $\noisyx \sim p(\noisyx \mid \vx)$, $\noisyx \in \tilde{\mathcal{X}}$ denote the corresponding noisy data where $p(\noisyx \mid \vx)$ is the noise distribution that describes the corruption process, e.g., $p(\noisyx \mid \vx) = \mathcal{N}(\noisyx \mid \vx, \sigma^2 I)$.
The goal of denoising is to learn a function, $g_{\bm{\theta}}: \tilde{\mathcal{X}} \rightarrow \mathcal{X}$, with parameters $\bm{\theta}$, that can recover $\vx$ from $\noisyx$, i.e., $\hat{\vx}_{\bm{\theta}} = g_{\bm{\theta}}(\noisyx)$.
Given a dataset of noisy and noiseless pairs, $\mathcal{D} = \{(\noisyx_i, \vx_i) \}_{i=1}^N$, we train $g_{\bm{\theta}}$ by minimizing a loss function that is appropriate for the data modality.

EHRs are commonly represented as binary vectors, $ \noisyx = [\tilde{x}_1, \ldots, \tilde{x}_d, \ldots, \tilde{x}_T]$, and $ \vx = [x_1, \ldots, x_d, \ldots, x_T]$, where $\tilde{x}_d, x_d \in \{0, 1\}$.
Given an ICD code from the noisy patient EHR, $\tilde{x}_d$, there are two possible scenarios when performing imputation of unknown unknowns.
If $\tilde{x}_d = 1$, then we preserve it as we assume that there is no ambiguity for positive diagnosis.
If $\tilde{x}_d=0$, then denoising is done to infer whether the 0 is a negative diagnosis or a missing positive diagnosis.

When a clean EHR dataset is available, we propose to train a transformer-based denoising neural network, $g_{\bm{\theta}}$, taking noisy EHR of a patient as input during training, which can be obtained by masking the clean data, and predicting the complete EHR. The transformer-based denoising neural network learns the correlation between different dimensions of EHRs in order to predict the true missing values in EHRs; we call our proposed approach Denoise2Impute.

We begin by discussing the motivation for the choice of the Set Transformer for parametrizing $g_{\bm{\theta}}$.
As mentioned previously, a 0 for an ICD code in the noisy EHR could mean that the same ICD code in the noiseless EHR could either be a 0 or a 1.
Thus, to perform well, it is crucial for the network to learn and leverage relationships between diagnosis codes.
As an example, suppose that in a patient's EHR we see there is a 0 for the diabetes ICD code.
Patients with obesity and high glucose are also likely to have diabetes.
Thus, if a 1 is present for both the obesity and high glucose ICD codes, 
but a 0 is present for the diabetes code, then it is likely that there is a missing diabetes diagnosis.
In contrast, if the patient has a clean bill of health and there is a 0 present for the diabetes code, it is likely that there is not a missing diabetes diagnosis.

For learning this relationship, we parameterize $g_{\bm{\theta}}$ with a Set Transformer~\citep{lee2019set}, as they have been shown to be able to learn pairwise and higher-order interactions between the dimensions of an input.
Moreover, unlike standard transformers, the Set Transformer architecture is agnostic to the order of the input, alleviating the need to design a positional encoding scheme for ICD codes (which is non-trivial).

For binary data, a natural choice for the loss function is the element-wise cross entropy (CE) loss
\begin{equation}\label{eq:first_attempt}
    \ell_{CE}(g_{\bm{\theta}}(\noisyx_i), \vx_i) = \sum_{d=1}^T \ell_{ce}(g_{\bm{\theta}}(\noisyx_i)_d, x_{i, d}),
\end{equation}
where $g_{\bm{\theta}}(\noisyx_i)_d$ is the $d$-th dimension of $g_{\bm{\theta}}(\noisyx_i)$ and
\begin{equation*}
    \begin{split}
        \ell_{ce}(g_{\bm{\theta}}(\noisyx_i)_d, x_{i, d}) = &x_{i, d} \log g_{\bm{\theta}}(\noisyx_{i})_d \\
         &+ (1 - x_{i, d}) \log (1 - g_{\bm{\theta}}(\noisyx_{i})_d).
    \end{split}
\end{equation*}
We found that when using~\eqref{eq:first_attempt} for training, the network underperformed on examples where
 $\tilde{x}_{i, d}=0$ and $x_{i, d} = 1$. 
Thus, we modified the loss function to put more weight on these instances
\begin{equation*}
    \label{eq:final_loss_function}
    \begin{split}
    \tilde{\ell}_{CE}(g_{\bm{\theta}}(\noisyx_i), \vx_i; \lambda) = &\sum_{d=1}^T  \indicator_{[\tilde{x}_{i,d} = x_{i,d}]} \ell_{ce}(g_{\bm{\theta}}(\noisyx_i)_d, x_{i, d}) \\
    & + \lambda \indicator_{[\tilde{x}_{i,d} \neq x_{i,d}]} \ell_{ce}(g_{\bm{\theta}}(\noisyx_i)_d, x_{i, d}),
    \end{split}
\end{equation*}
where $\lambda$ is a hyperparameter that dictates how much emphasis we should place on training examples where $\tilde{x}_{i, d}=0$ and $x_{i, d} = 1$. 

\section{Denoise2Impute-T: A theoretically motivated extension}
While Denoise2Impute was generally effective at imputing unknown unknowns, we found instances where it struggled to discern whether a $0$ was missing or a negative diagnosis.
We hypothesized that this is in part due to the difficulty of the problem.
Classically, when performing denoising we often assume--implicitly or explicitly--that we can discriminate between noiseless data, $\vx$, and noisy data, $\noisyx$.
Intuitively, this makes sense; we only perform denoising if the data is noisy, i.e., $\noisyx \neq \vx$.
If a data point is not noisy, i.e., $\noisyx = \vx$, then no denoising is performed, i.e., $g_{\bm{\theta}}$ is the identity function.

In the \textit{unknown unknowns} scenario, it is no longer easy to discriminate between $\noisyx$ and $\vx$.
For example, in the case of EHRs it is difficult to discriminate between $\noisyx$ and $\vx$ as 1) EHRs are generally very sparse and 2) the absence of a code can either mean that the patient was not diagnosed with that disease or that the patient was diagnosed but the code was not recorded or is unavailable in the EHR.
Thus, given a sparse EHR vector it is difficult to discern whether it is noiseless or whether it should be denoised.
Intuitively, we would expect denoising in the unknown unknowns scenario to be difficult as now we must discern which points need to be denoised.
While in principle, we could denoise every 0, this may introduce artifacts into the data, especially when the data are sparse~\citep{tanner2015inappropriate} and lead to suboptimal performance.

To formalize the intuition presented above, we begin by restricting ourselves to the case where $\mathcal{X} \subseteq \mathbb{R}^T$ and $\tilde{\mathcal{X}} \subseteq \mathbb{R}^T$ and write the noise distribution, $p(\noisyx \mid \vx)$, as follows:
\begin{equation}\label{eq:new_noise_corruption}
    p( \noisyx \mid \vx) = \beta \delta(\noisyx = \vx) + (1 - \beta)  q(\noisyx \mid \vx),
\end{equation}
where $\delta(\cdot)$ is the Dirac delta function and $\beta \in [0, 1]$.
With~\eqref{eq:new_noise_corruption} we attempt to model the unknown unknowns setting, where with probability $\beta$ no noise is added, i.e., $\noisyx = \vx$, and with probability $1 - \beta$ noise is introduced via $q(\noisyx \mid \vx)$.
\eqref{eq:new_noise_corruption} demonstrates the difficulty of the problem: if $\beta$ is large then with high probability, $\noisyx = \vx$, making it impossible to discriminate between the two.
Moreover, the form of $q(\noisyx \mid \vx)$ can also make the problem difficult as noise can be injected that can make $\noisyx$ likely under the data distribution, $p(\vx)$.

Given~\eqref{eq:new_noise_corruption}, we seek to derive the optimal denoising function, $f_\star: \tilde{\mathcal{X}} \rightarrow \mathcal{X}$.
Specifically, we seek the optimal denoising function
that minimizes the expected mean squared error (MSE)
\begin{equation*}
    \label{eq:variational_denoising}
    f_\star(\cdot) = \underset{f(\cdot)}{\textrm{argmin}} \quad \mathbb{E}_{p(\vx) p(\noisyx \mid \vx)}\left[ (f(\noisyx) - \vx)^\top (f(\noisyx) - \vx) \right].
\end{equation*}
Using calculus of variations, we can derive the functional form of $f_\star$. 
Due to space, we state an informal theorem below, which we prove in the Appendix.

\begin{theorem}\label{thm:optimal_denoising}
    Let $p(\vx)$ be the data distribution, let $p(\noisyx \mid \vx)$ the noise distribution as defined in~\eqref{eq:new_noise_corruption} and let $q(\noisyx) = \int q(\noisyx \mid \vx) p(\vx) d\vx$ the marginal distribution of the noisy data.
    With respect to MSE, the optimal denoising function is
    \begin{equation}
        \label{eq:optimal_f_star}
        f_\star(\noisyx) = \omega(\noisyx) \noisyx + \left( 1 - \omega(\noisyx)\right) \, g_\star(\noisyx),
    \end{equation}
    where
    \begin{gather*}
        g_\star(\noisyx) \triangleq \mathbb{E}_{q(\vx \mid \noisyx )}[\vx],
        \omega(\noisyx) \triangleq \frac{p(\noisyx)}{p(\noisyx) + \gamma \tilde{q}(\noisyx)},
        \gamma \triangleq \frac{1 - \beta}{\beta}.
    \end{gather*}
\end{theorem}
From~\eqref{eq:optimal_f_star}, we see that the optimal denoising function $f_\star$ is a convex combination of the identity function and the minimum MSE, $g_\star$~\citep{casella2021statistical}, where the weight is determined by $\omega(\noisyx) \in [0, 1]$.
Specifically, $\omega(\noisyx)$ measures how likely $\noisyx$ is under the data distribution $p(\vx)$, compared to the scaled marginal distribution of the noisy data $\gamma \tilde{q}(\noisyx)$.
Denoising known unknowns corresponds to the case where it is easy to discriminate between $p(\vx)$ and $\gamma \tilde{q}(\noisyx)$, corresponding to $\omega(\noisyx) \approx 0$ and $\omega(\noisyx) \approx 1$.
Denoising unknown unknowns corresponds to the scenario where it is hard to discriminate between $p(\vx)$ and $\tilde{q}(\noisyx)$, corresponding to $p(\vx) \approx \tilde{q}(\vx)$;
this leads to $\omega(\noisyx) \approx \frac{1}{1 + \gamma}$.
Thus, $\gamma$ can be interpreted as a regularizer that allows us to choose how much weight $f_\star$ puts on the identity function or on the MMSE estimator which can lead to better behaved estimators.

Inspired by Theorem~\ref{thm:optimal_denoising}, we aim to extend Denoise2Impute.
We cannot straightforwardly use Theorem~\ref{thm:optimal_denoising} as (i) it requires knowledge of $p(\vx)$, $\beta$, and $q(\noisyx \mid \vx)$, which is usually not known and (ii) is designed for the MSE loss function, which assumes real-valued data, while we consider binary EHR data in this study.
However, we conjecture that for many popular loss functions, $\ell$, the optimal denoising function can be written as
\begin{equation}\label{eq:conjectured_optimal_denoising}
    f_\star(\noisyx) = \omega(\noisyx) \noisyx + \left( 1 - \omega(\noisyx)\right) \, g_\ell(\noisyx),
\end{equation}
where the functional form of $g_\ell: \tilde{\mathcal{X}} \rightarrow \mathcal{X}$, depends on the choice of $\ell$, i.e, if $\ell$ is MSE then $g_\ell$ becomes the MMSE.

Inspired by~\eqref{eq:conjectured_optimal_denoising}, we can extend Denoise2Impute, $g_{\bm{\theta}}$, by introducing a function that determines which codes to denoise.
Specifically, given a trained denoiser, $g_{\bm{\theta}}$, and a function, $\omega_{\bm{\phi}}$, we can combine them to create a new denoising function, $f$, which we can be expressed as
\begin{equation}\label{eq:practical}
    f(\noisyx) = \omega_{\bm{\phi}} (\noisyx) \noisyx + \left( 1 - \omega_{\bm{\phi}} (\noisyx) \right) g_{\bm{\theta}} (\noisyx).
\end{equation}
We call~\eqref{eq:practical} Denoise2Impute-T.

While there are many ways to parameterize $\omega_{\bm{\phi}}$, we found the following to work well
\begin{equation}\label{eq:threshold}
    \omega_{\bm{\phi}}(\noisyx) = \indicator_{[g_{\bm{\psi}}(\noisyx) < \bm{\phi}]},
\end{equation}
where $\indicator_{[\cdot]}$ is the indicator function and $\bm{\phi} = [\phi_1, \ldots, \phi_T]$ are learnable thresholds.
The use of $\bm{\phi}$ allows Denoise2Impute-T to adaptively choose when to use the identity function or $g_{\bm{\psi}}$ for each individual dimension of $\noisyx$.
While the output of $\omega$ in~\eqref{eq:optimal_f_star} is a scalar, we found that making it vector-valued, i.e, having a separate threshold for each dimension of $\noisyx$, led to better empirical performance.

To train $\omega_{\bm{\phi}}$ we use a continuous relaxation of the indicator function and use~\eqref{eq:final_loss_function} to compute the loss with respect to $f$~\eqref{eq:practical}
\begin{equation*}
    \begin{split}
    \tilde{\ell}_{CE}(f(\noisyx_i), \vx_i; \lambda) = &\sum_{d=1}^T  \indicator_{[x_{i,d} = \tilde{x}_{i,d}]} \ell_{ce}(f(\noisyx_i)_d, x_{i, d}) \\
    & + \lambda \indicator_{[x_{i,d} \neq \tilde{x}_{i,d}]} \ell_{ce}(f(\noisyx_i)_d, x_{i, d}),
    \end{split}
\end{equation*}
where $g_{\bm{\theta}}$ is kept fixed and only the parameters of $\omega_{\bm{\phi}}$ are updated.
More details in training can be found in Appendix~\ref{app:training_details}.
\section{Results}\label{exp_details}
We include comprehensive experiments to demonstrate the accuracy and utility of our proposed methodology.
First, we evaluate the denoising accuracy on a large EHR dataset.
Next, the denoised EHRs are used in downstream prediction tasks.
Below we provide brief details on the datasets and baselines.
More details regarding data sources, model architectures and hyperparameters can be found in Appendix~\ref{app:supp_experiments}.

\paragraph{Datasets}
Payer companies and care delivery organizations (CDOs) commonly have overlapping patient population but have unique coding practices, i.e., the same patient's EHR can differ between the payer company and the CDO.
Thus, for training the denoising model, we begin with two sources of EHR data with overlapping patient populations from a large healthcare institute, $\mathcal{D}_1 = \{\tilde{\mathbf{x}}_1, \ldots, \tilde{\mathbf{x}}_{N}\}$ and $\mathcal{D}_2 = \{\tilde{\mathbf{y}}_1, \ldots, \tilde{\mathbf{y}}_M\}$. 
Each data point is a vector of binary indicators for the presence of ICD-10 diagnosis codes in a patient's EHR, where we set the number of codes to be 993, i.e., $T=993$. We leverage the overlapping patient populations of $\mathcal{D}_1$ and $\mathcal{D}_2$ to construct a dataset $\mathcal{D}$ of size 445,345 patients. More details can be found in Appendix~\ref{app:datasets}.


In Figure~\ref{fig:dataset_details}, we provide descriptive statistics for~$\mathcal{D}_1$~and~$\mathcal{D}_2$.
Figure~\ref{fig:dataset_details}A demonstrates that a majority of the 993 diagnosis codes have a prevalence of less than $\sim3\%$ for both $\mathcal{D}_1$ and $\mathcal{D}_2$. Additionally, codes tend to have higher prevalence in $\mathcal{D}_2$ than in $\mathcal{D}_1$.
Figure~\ref{fig:dataset_details}B demonstrates the sparsity of patient EHRs, where for both $\mathcal{D}_1$ and $\mathcal{D}_2$, a majority of patient EHRs have less than 6\% of the ICD codes present. Furthermore, since each dot corresponds to a patient whose EHR is analyzed in $\mathcal{D}_1$ and $\mathcal{D}_2$, we can compare the sparsity across the two datasets for each patient. We can therefore conclude that $\mathcal{D}_1$ has a higher degree of sparsity in patient EHRs than $\mathcal{D}_2$.

While both $\mathcal{D}_1$ and $\mathcal{D}_2$ are very sparse, Figure~\ref{fig:dataset_details}C and \ref{fig:dataset_details}D demonstrate a peak in eigenvalues of the covariance matrix of $\mathcal{D}_1$ and $\mathcal{D}_2$ compared to eigenvalues of the covariance matrix of random binary matrices (with prevalence matched to $\mathcal{D}_1$ and $\mathcal{D}_2$, respectively). This suggests that $\mathcal{D}_1$ and $\mathcal{D}_2$ comprise higher-order structure across their dimensions which can then be leveraged by imputation methods to learn relationships between ICD-10 codes. Therefore, $\mathcal{D}_1$ and $\mathcal{D}_2$ provides versatile benchmarks to evaluate different methods.

In the final, we employ a real world dataset to evaluate our method on a popular healthcare prediction task from the literature, predicting hospital readmission from ICD-10 codes~\citep{kansagara2011risk, huang2019clinicalbert, mahmoudi2020use, jiang2023health}. The size of the dataset is 1,670,347 of which the size of training dataset is 1,124,550 and the size of the testing dataset is 545,797. For the hospital readmission task, the trained denoising model is used to denoise a population of 1,670,347 patients which comes from the same source as $\mathcal{D}_1$.
\begin{figure*}[ht!]
  \centering
  \includegraphics[width=0.7\textwidth]{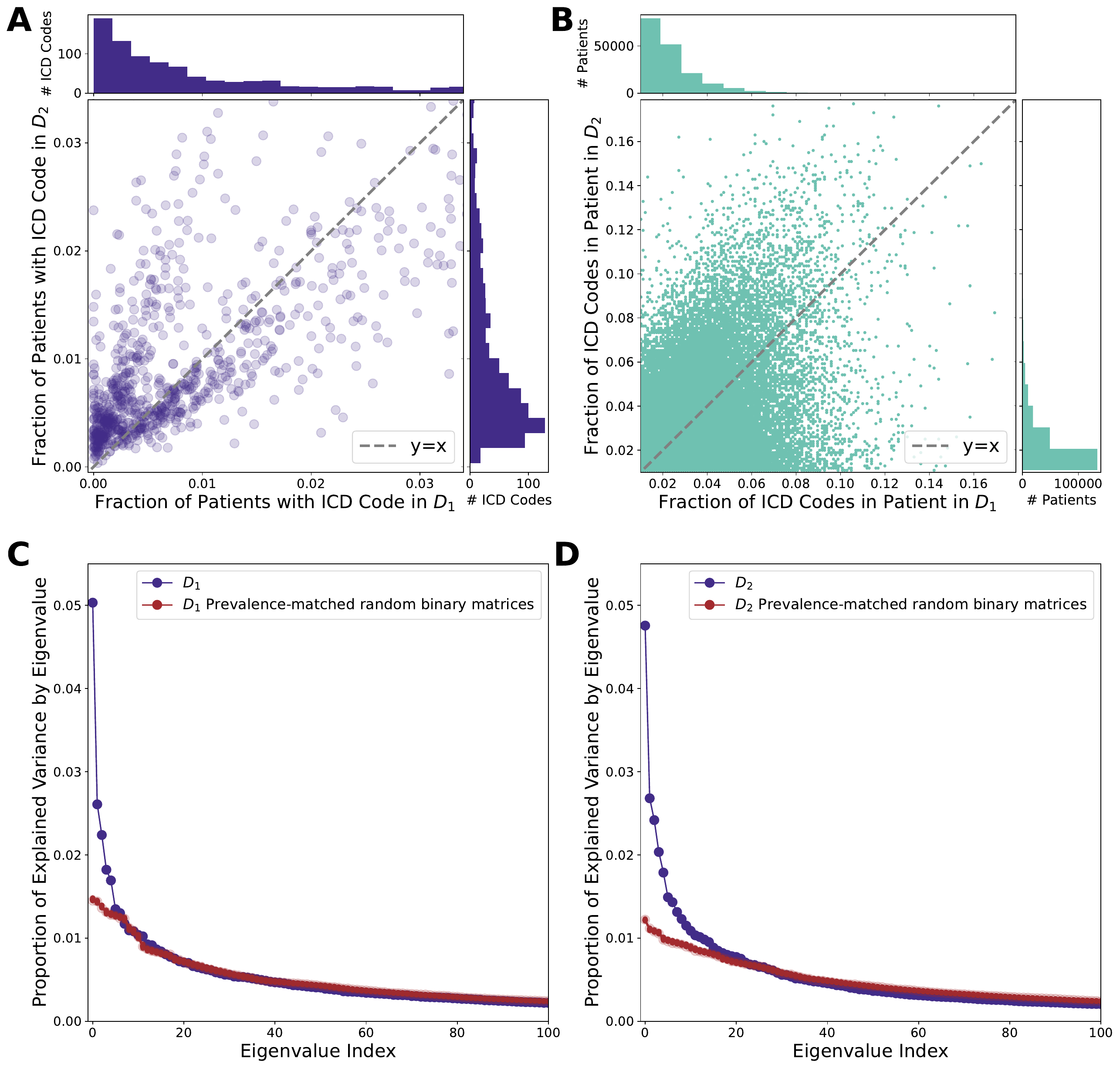}
  \caption{
    Descriptive statistics for both datasets used in this paper.
    \textbf{A)} Comparison of the prevalence of ICD codes in $\mathcal{D}_1$ vs $\mathcal{D}_2$.
    \textbf{B)} Comparison of the sparsity of patient EHRs in $\mathcal{D}_1$ vs $\mathcal{D}_2$.
    \textbf{C)} The eigenvalue spectra of the covariance matrix of $\mathcal{D}_1$ (blue) with respect to 100 prevalence-matched random binary matrices (red).
    \textbf{D)} The eigenvalue spectra of the covariance matrix of $\mathcal{D}_2$ (blue) with respect to 100 prevalence-matched random binary matrices (red).
  }
  \label{fig:dataset_details}
\end{figure*}

\paragraph{Baselines} 
We compare our two proposed approaches, Denoise2Impute and Denoise2Impute-T,
against popular imputation methods
wherein, for the imputation approaches all 0's are treated as missing values.
Specifically, we compare against the following imputation approaches: 1) replacing all 0's with the \textbf{prevalence} computed by calculating mean for each dimension in $\mathcal{D}_1$,
2) k-nearest neighbor \textbf{(k-NN)} (with k = 5) and 3) \textbf{softImpute}~\citep{mazumder2010spectral} using the implementation from~\citet{fancyimpute}.
We also compare the performance of Denoise2Impute but parametrized $g_{\bm{\theta}}$ with a feedforward neural network as opposed to a Set Transformer; this is denoted as \textbf{MLP}.
We also compare against a denoising auto-encoder \textbf{(DAE)} and the collaborative denoising auto-encoders \textbf{(CDAE)} from~\citet{wu2016collaborative}.
We note that we attempted to compare against the recommender-system based approach of~\citep{suleiman2020clinical}, but their code was not publicly available.
Training and architecture details can be found in Appendix~\ref{app:training_details}.


\begin{table*}[ht!]
  \centering
  \begin{tabular}{ |c | c | c | c |c | c | c | c |}
    \hline
    Methods  & $\mathcal{D}_1$ & Prevalence & MLP & DAE & CDAE & D2Impute & D2Impute-T \\ \hline
    Avg. AUPRC & 51.86 & 52.37 & 61.91 & 62.05 & 62.17 & 64.90 & {\bf 65.53}  \\ \hline
    \end{tabular}  
\captionof{table}{Average of dimension-wise AUPRC (higher is better) on the test set where noisy vectors were obtained from $\mathcal{D}_1$. D2Impute is short for Denoise2Impute.}
\label{tab:test_set_results}
\end{table*}

\paragraph{Benchmarking Imputation Performance} To investigate the imputation
performance of all the methods, we compute the AUPRC between the imputed values
 and $\mathcal{D}$ for each dimension. 
Table~\ref{tab:test_set_results} summarizes the AUPRC averaged across all 993 dimensions. 
We observe that both Denoise2Impute and Denoise2Impute-T outperform the other methods, with Denoise2Impute-T achieving the
the highest AUPRC and also improves upon using the original noisy data, $\mathcal{D}_1$.

\begin{figure*}[ht!]
  \begin{center}
  \begin{tabular}{ccc}
  \includegraphics[width=0.28\textwidth]{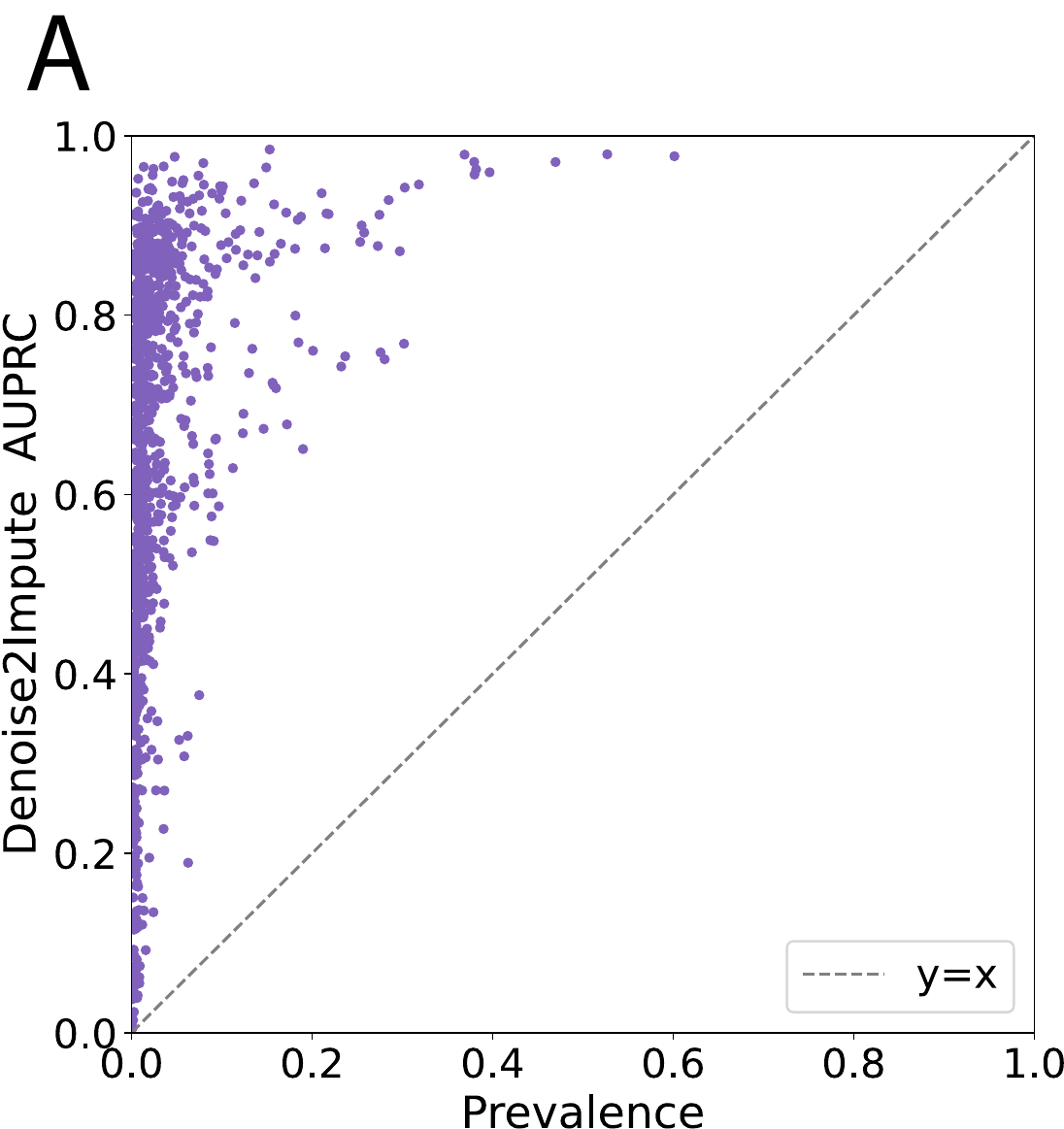} &
  \hspace{-0.5cm}
  \includegraphics[width=0.28\textwidth]{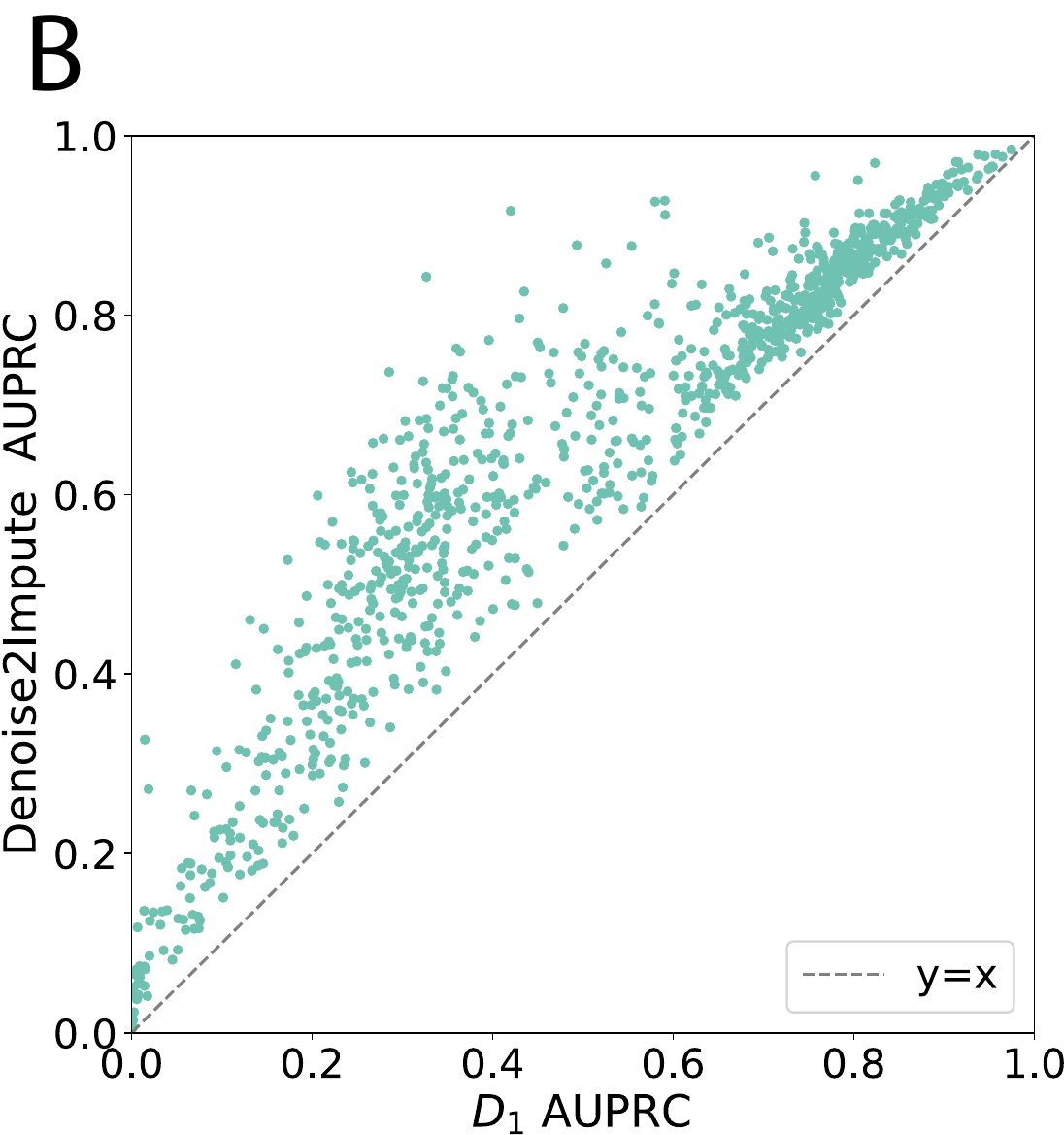} &
  \hspace{-0.5cm}
  \includegraphics[width=0.28\textwidth]{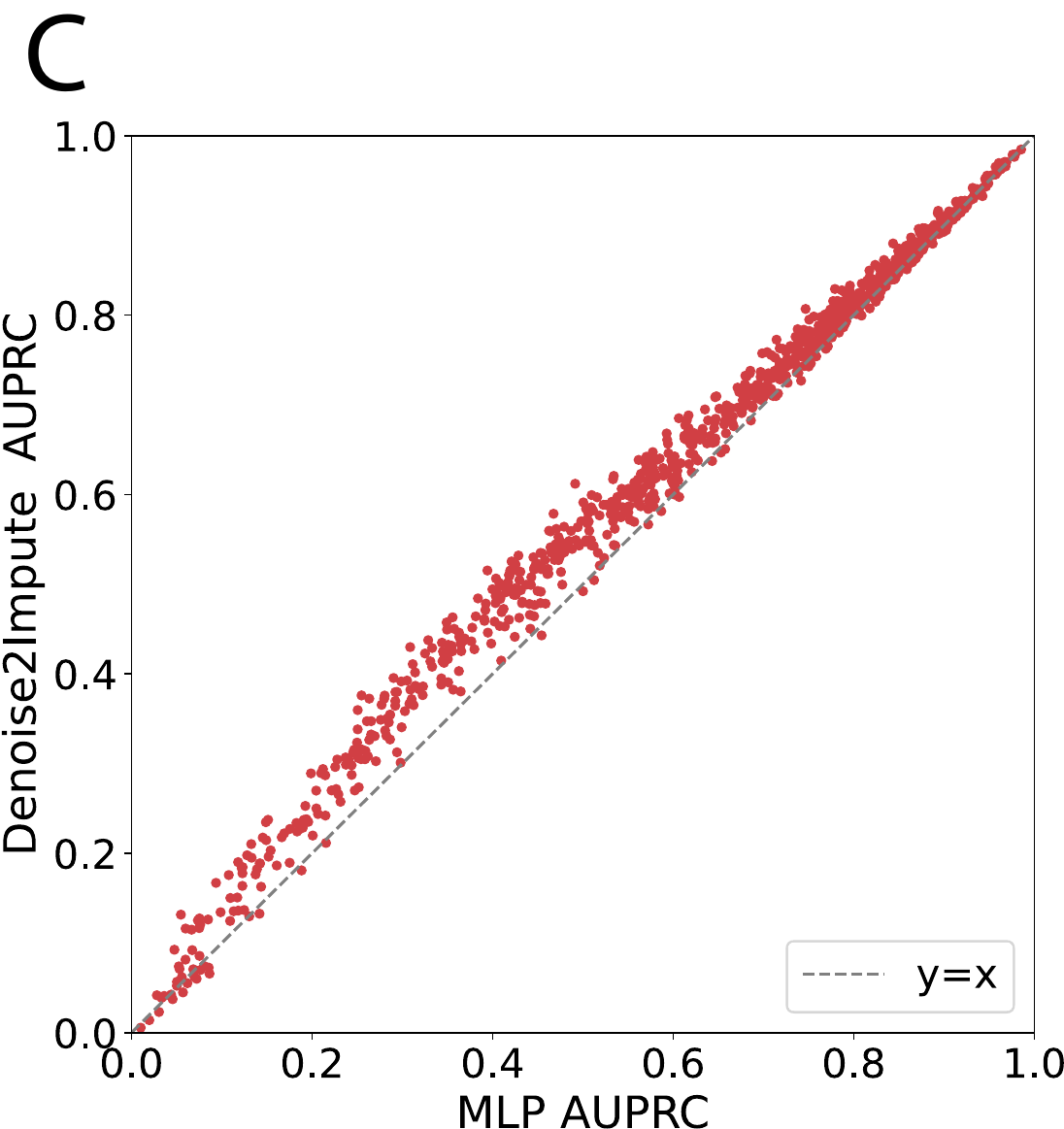} \\
  \end{tabular}
  \caption{Dimension-wise AUPRC for $T=993$ computed on the test data. \textbf{A)} Comparison between AUPRC of the Denoise2Impute and the column mean (\textit{prevalence}) for each dimension of the noisy data $\mathcal{D}_1$. \textbf{B)} Comparison between AUPRC of the Denoise2Impute and that of the $\mathcal{D}_1$. \textbf{C)} Comparison between AUPRC of the Denoise2Impute and AUPRC of the MLP.}
  \label{fig2}
  \end{center}
\end{figure*}

In Figure~\ref{fig2}, we investigate the performance of Denoise2Impute for denoising individual ICD codes.
Specifically, we compare the dimension-wise AUPRC of the Denoise2Impute method against the prevalence
 ~(Figure~\ref{fig2}A), not applying imputation and using $\mathcal{D}_1$ as is (Figure~\ref{fig2}B), and an MLP-parameterized Denoise2Impute~(Figure~\ref{fig2}C).
 In Figure~\ref{fig2}A, we observe that the AUPRC of the denoised data using the Denoise2Impute is higher than the column mean (also referred to as \textit{prevalence}) for each of the 993 dimensions. 
In Figure~\ref{fig2}B, we observe that Denoise2Impute achieves higher AUPRC across all 993 dimensions than using the noisy data $\mathcal{D}_1$ itself as the denoised data.
In Figure~\ref{fig2}C, we observe that using a Set Transformer to parametrize $g_\theta$ in Denoise2Impute achieves higher AUPRC than using a feedforward neural network (referred to as MLP) in 866 ($\sim87\%$) out of the 993 dimensions. 

Furthermore, to investigate how the noise distribution affects the performance of models, we created a synthetic noisy dataset from $\mathcal{D}$, where the noise was distributed such that the prevalence of the synthetic noisy dataset matches that of $\mathcal{D}_1$. The result on synthetic data in Figure~{\ref{fig:synthetic_example}} in Appendix~\ref{app:synthetic} shows consistent results with that of Figure~\ref{fig2}. Together, both real and synthetic data experiments provide empirical support for the use of a Set Transformer within Denoise2Impute.

\begin{figure*}[ht!]
  \centering
    \begin{tabular}{ccc}
    \includegraphics[width=0.333\textwidth]{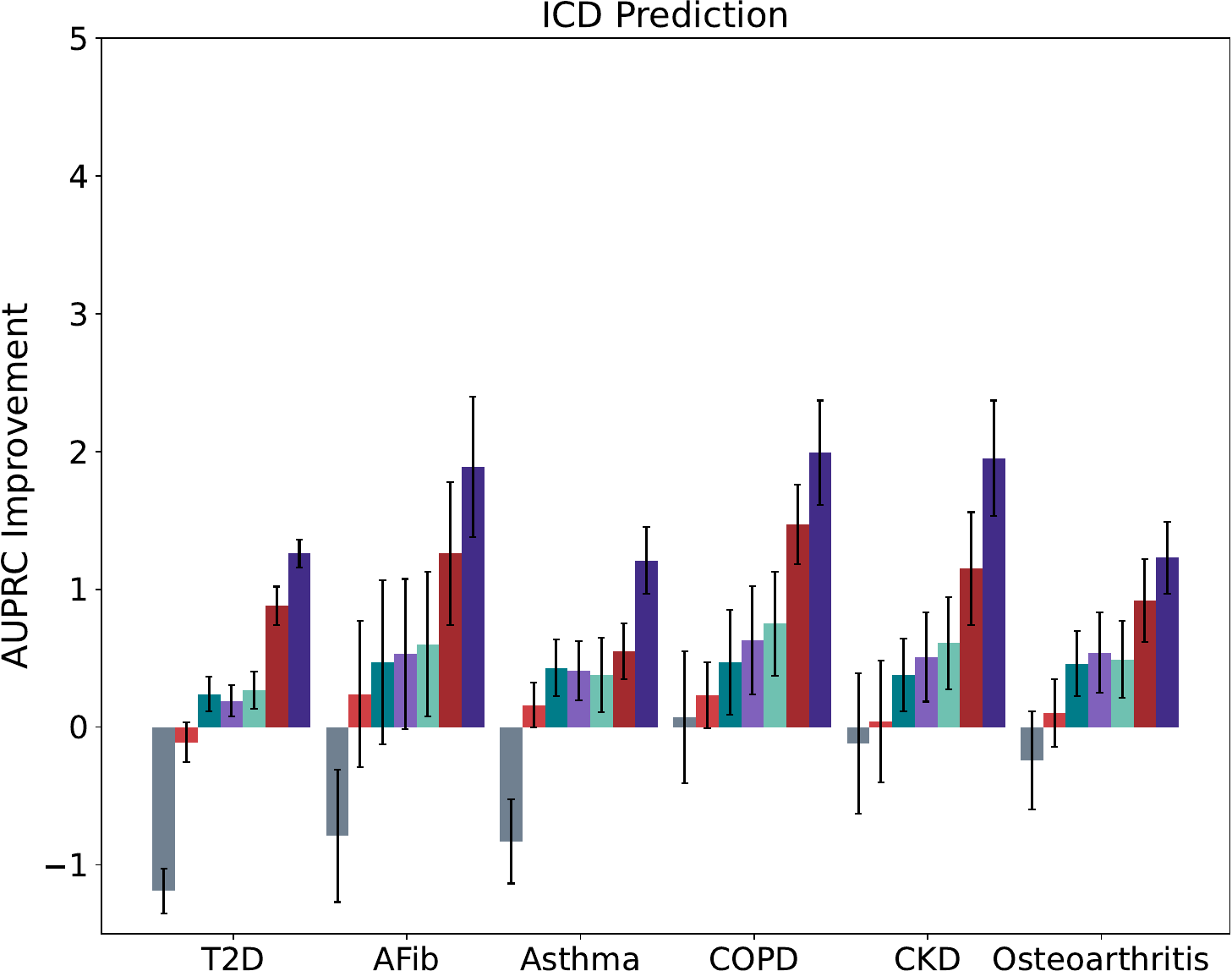} &
    \hspace{-0.5cm}
    \includegraphics[width=0.325\textwidth]{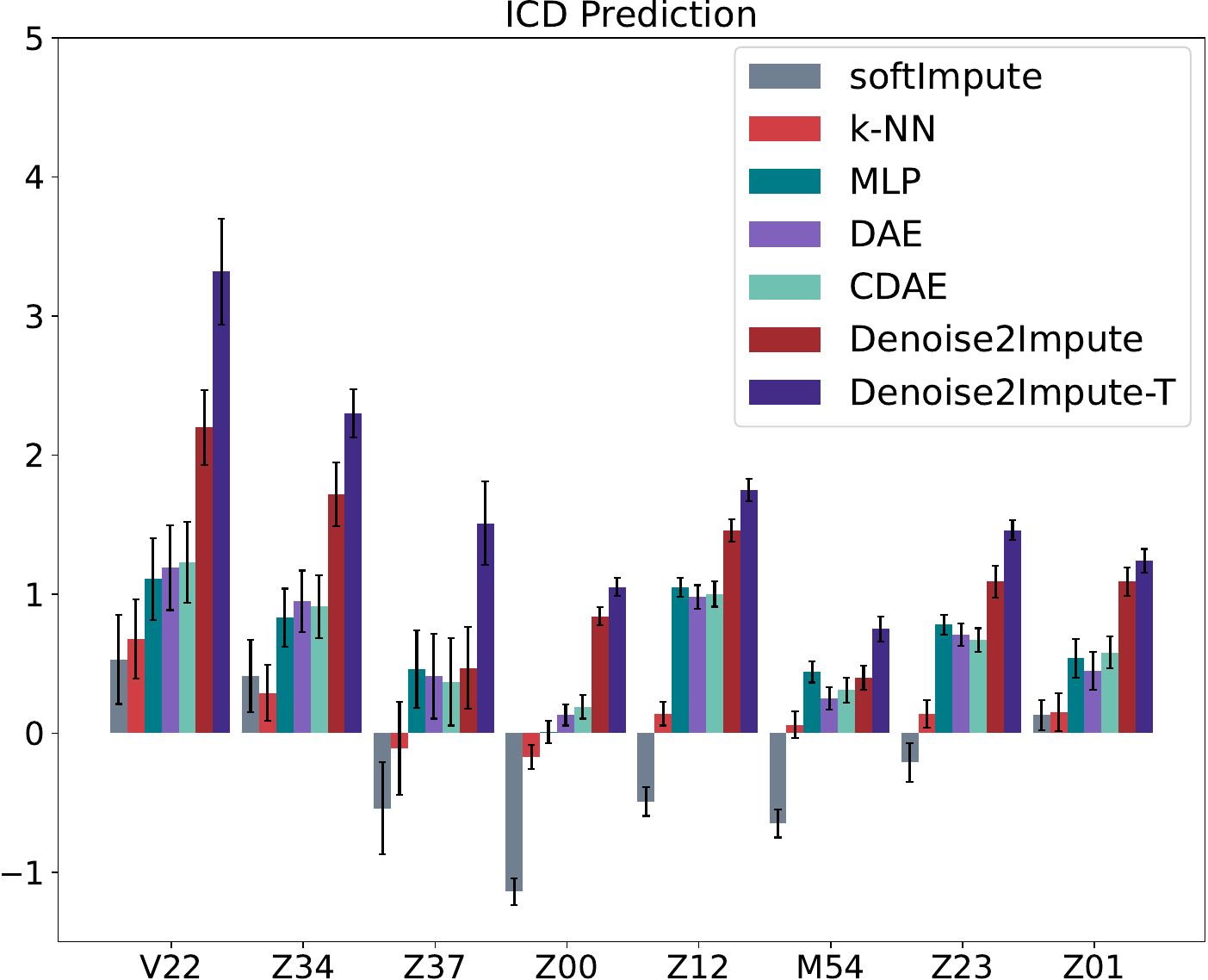} &
    \hspace{-0.5cm}
    \includegraphics[width=0.333\textwidth]{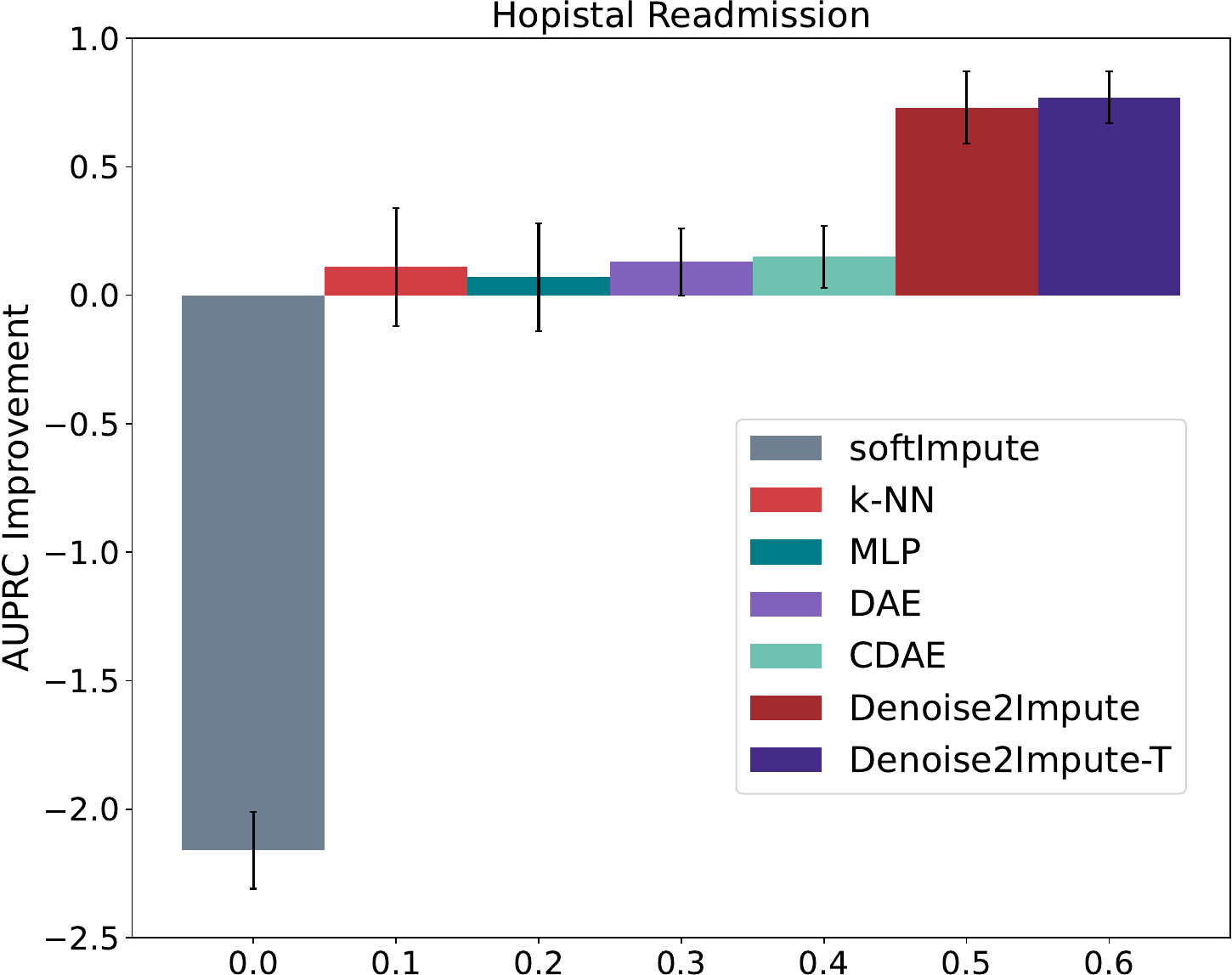} \\
    {(A) Diseases} &   {(B) ICD codes} & {(C) Downstream Task}
    \end{tabular}
  \caption{Results for ICD code prediction for six common chronic conditions (A) and eight randomly chosen ICD codes (B). Results for Hospital Readmission task (C). 
  Difference in AUPRC relative to using $\mathcal{D}_1$ and 95\% confidence intervals are plotted. The common chronic conditions in (A) and the description of ICD codes in (B) are provided in Table~\ref{app:tab:eight_disease_captions} in Appendix~\ref{app:icd_details}. }
  \label{fig:icd_prediction}
\end{figure*}

\paragraph{Benchmarking Imputation on Downstream Tasks} We investigate how imputation
performance can affect downstream task prediction, by benchmarking on two downstream tasks.
Specifically, we first use the trained Denoise2Impute models to denoise the noisy data $\mathcal{D}_1$ and then feed the denoised values into a downstream classifier~\cite{ke2017lightgbm}.

To investigate whether the proposed approach was able to learn the relationship between diagnosis codes, we aim to predict a particular ICD code among the $T$ codes given the others.
To avoid information leakage, the target ICD code is set to 0 in the patient EHR before it is passed to the imputation methods.
We begin by selecting six ICD codes for common chronic diseases and compute the difference in AUPRC with respect to using the original data, $\mathcal{D}_1$, along with 95\% confidence intervals~(Figure~\ref{fig:icd_prediction}A).
We observe that both Denoise2Impute and Denoise2Impute-T outperform all other methods for all six diseases, where Denoise2Impute-T tends to outperform Denoise2Impute.
Interestingly, we found that softImpute often led to worse performance compared to using the original data, $\mathcal{D}_1$.
The performance of softImpute suggests that treating all 0s as missing and performing imputation may lead to suboptimal performance in downstream tasks.
In Figure~\ref{fig:icd_prediction}B, we repeated the above experiment for eight additional randomly chosen ICD codes where we see the same trend as before; the description for the ICD codes can be found in Table~\ref{app:tab:eight_disease_captions} in Appendix~\ref{app:icd_details}.

Finally, we evaluate the performance on predicting hospital readmission, which is a popular clinical prediction task in the literature~\citep{jiang2023health, kansagara2011risk, huang2019clinicalbert, mahmoudi2020use}. While ICD-10 codes can be used to predict hospital readmission, the fact that readmission prediction is not an ICD-10 code itself enhances its credibility as a benchmarking dataset since we no longer need to set a dimension to all 0s before performing denoising. In Figure~\ref{fig:icd_prediction}C, we observe that both Denoise2Impute and Denoise2Impute-T
achieve the highest AUPRC compared to the other methods in prediction of hospital readmission and that this increase is statistically significant based on the $95\%$ confidence intervals.
Moreover, we see that the Denoise2Impute-T leads to a slight performance gain compared to Denoise2Impute but that this improvement is not statistically significant.
Taken together, Denoise2Impute and Denoise2Impute-T outperform all existing imputation approaches on both, imputation and prediction tasks.

\section{Discussion}
Missing data presents a substantial challenge for machine learning in the medical domain, as it can negatively impact performance and perpetuate existing disparities in healthcare~\citep{vyas2020hidden,pierson2021algorithmic}. The problem becomes even more complex when the presence of missing entries is unclear. In response to this challenge, we introduced Denoise2Impute, a denoising-based approach that utilizes relationships between diagnosis codes to identify and impute missing codes effectively.

While our empirical evidence highlights Denoise2Impute's effectiveness, we observed that certain diagnosis codes posed challenges for the approach. 
These difficulties may arise due to the extreme sparsity of some codes or insufficient training examples, making it hard for the model to learn useful dependencies between diagnosis codes. 
In such cases, Denoise2Impute's output could introduce noise into the data, as the network struggles to differentiate between missing and existing diagnosis codes. 
To enhance the model's performance, we developed a theoretically motivated extension called Denoise2Impute-T, which employs learned thresholds to regularize the model's output.
Our empirical results show that using learned thresholds with Denoise2Impute-T leads to performance gains over the original Denoise2Impute.

Despite these encouraging findings, there are limitations to our approach and opportunities for improvement. 
First, training Denoise2Impute necessitates a noiseless and noisy training set; while this may seem challenging, their are a number of ways to obtain such a dataset.
First, insurance companies and care delivery organizations (CDOs) commonly have overlapping patient population where an EHR for the same patient can differ between the insurance company and the CDO.
Having access to these two dataset can allow for the creation of a noiseless dataset which can be used to train the model.
Another approach is to create an artificial noisy dataset given a dataset of EHR's via 0-1 binary masking, which is commonly employed in large language models~\citep{devlin2018bert}.
Lastly, the approach of~\citep{wu2023collecting} can be used to allow for a cost-effective construction of a dataset of noisy and clean EHRs. 
Additionally, incorporating other tabular information, such as medications and laboratory results, could further enhance performance. 
Expanding the model to include multiple data modalities, like clinical text and the time-series component of EHR data, could provide a more comprehensive representation of patient history and lead to improved performance. 
For example, the threshold function $\omega_{\bm{\phi}}$ could be refined by considering additional information sources, allowing for patient-specific thresholding.

\bibliography{ref.bib}
\appendix
\onecolumn
\section{Proof of Theorem~\ref{thm:optimal_denoising}}\label{app:proof}
For ease of presentation, we rewrite the theorem below.
\begin{theorem}
  Let $p(\vx)$, $\vx \in \mathcal{X} \subseteq \mathbb{R}^T$, be the distribution over the data and let $p(\noisyx \mid \vx)$,~$\noisyx \in \tilde{\mathcal{X}} \subseteq \mathbb{R}^T$,  be the noise distribution, which we assume is of the following form
  \begin{equation}
    p( \noisyx \mid \vx) = \beta \delta(\noisyx = \vx) + (1 - \beta)  q(\noisyx \mid \vx) .
  \end{equation}
  Moreover, let $q(\noisyx)$ be
  \begin{equation}
      q(\noisyx) = \int q(\noisyx \mid \vx) p(\vx) d\vx ,
  \end{equation}
  be the marginal distribution of the noisy data.
  With respect to mean squared error, the optimal denoising function is
  \begin{equation}
    f_\star(\noisyx) = \omega(\noisyx) \noisyx + \left( 1 - \omega(\noisyx)\right) \, g_\star(\noisyx)
  \end{equation}
  where
  \begin{gather}
      g_\star(\noisyx) \triangleq \mathbb{E}_{q(\vx \mid \noisyx )}[\vx], \\
      \omega(\noisyx) \triangleq \frac{p(\noisyx)}{p(\noisyx) + \gamma \tilde{q}(\noisyx)}, \quad \omega(\noisyx): \tilde{\mathcal{X}} \rightarrow [0, 1],
  \end{gather}
  and $\gamma \triangleq \frac{1 - \beta}{\beta}$.
\end{theorem}
\begin{proof}
    To solve for $f_\star$, we use calculus of variations~\citep{liberzon2011calculus} but first, we will expand the loss function to make it easier to work with.
    We start by expanding the MSE and absorbing terms that are not a function of $f$ into a constant, $C$,
    \begin{align}
        &\int (f(\noisyx) - \vx)^\top (f(\noisyx) - \vx) p(\vx, \noisyx) d\vx d\noisyx \\
        & = \int f(\noisyx)^\top f(\noisyx) p(\vx, \noisyx) d\vx d\noisyx - 2 \int f(\noisyx)^\top \vx \, p(\vx, \noisyx) d\vx d\noisyx + \int \vx^\top \vx \, p(\vx, \noisyx) d\vx d\noisyx \\
        \label{eq:expanded_mse}
        & = \int f(\noisyx)^\top f(\noisyx) p(\vx, \noisyx) d\vx d\noisyx - 2 \int f(\noisyx)^\top \vx \, p(\vx, \noisyx) d\vx d\noisyx + C .
    \end{align}
    Dropping the constant, $C$, and using the chain rule of probability~\citep{casella2021statistical} to rewrite the joint distribution $p(\noisyx, \vx)$ as $p(\noisyx, \vx) = p(\vx) p(\noisyx \mid \vx)$, we have
    \begin{equation}\label{eq:proof_intermediate_eq}
        \mathcal{L}(f) \triangleq \textcolor{red}{\int f(\noisyx)^\top f(\noisyx) p(\vx) p(\noisyx \mid \vx) d\vx d\noisyx} - \textcolor{blue}{2 \int f(\noisyx)^\top \vx \, p(\vx) p(\noisyx \mid \vx) d\vx d\noisyx} .
    \end{equation}
    We will first simplify the \textcolor{red}{left term} of~\eqref{eq:proof_intermediate_eq} where we plug in the definition of the noise distribution, $p( \noisyx \mid \vx) = \beta \delta(\noisyx = \vx) + (1 - \beta)  q(\noisyx \mid \vx)$
    \begin{align}
        & \textcolor{red}{\int f(\noisyx)^\top f(\noisyx) p(\vx) p(\noisyx \mid \vx) d\vx d\noisyx} \\
        & = \textcolor{red} {\int f(\noisyx)^\top f(\noisyx) p(\vx) \left[ \beta \delta(\noisyx = \vx) + (1 - \beta) q (\noisyx \mid \vx) \right] d\vx d\noisyx} , \\
        & = \textcolor{red} {\beta \int f(\noisyx)^\top f(\noisyx) p(\vx) \delta(\noisyx = \vx) d\vx d\noisyx } + \textcolor{orange} { (1 - \beta) \int f(\noisyx)^\top f(\noisyx) p(\vx) q(\noisyx \mid \vx) d\vx d\noisyx} , \\
        &= \textcolor{red} {\beta \int f(\vx)^\top f(\vx) p(\vx=\noisyx) d\vx } + \textcolor{orange} { (1 - \beta) \int f(\noisyx)^\top f(\noisyx) q(\noisyx) d\noisyx} , \\
        \label{eq:proof_left_term}
        & = \textcolor{red} {\beta \int f(\noisyx)^\top f(\noisyx) p(\noisyx) d\noisyx } + \textcolor{orange} { (1 - \beta) \int f(\noisyx)^\top f(\noisyx) q(\noisyx) d\noisyx} .
    \end{align}
    We now move on to the \textcolor{blue}{right term} of~\eqref{eq:proof_intermediate_eq}
    \begin{align}
        & \textcolor{blue}{2 \int f(\noisyx)^\top \vx \, p(\vx) p(\noisyx \mid \vx) d\vx d\noisyx} \\
        & = \textcolor{blue}{2 \int f(\noisyx)^\top \vx \, p(\vx) \left[ \beta \delta(\noisyx = \vx) + (1 - \beta) q(\noisyx \mid \vx) \right] d\vx d\noisyx} , \\
        & = \textcolor{blue}{ 2\beta \int f(\noisyx)^\top \vx \, p(\vx)  \delta(\noisyx = \vx) d\vx d\noisyx } + \textcolor{purple}{2(1 - \beta) \int f(\noisyx)^\top \vx \, p(\vx)q(\noisyx \mid \vx) d\vx d\noisyx} , \\
        & = \textcolor{blue}{ 2\beta \int f(\vx)^\top \vx \, p(\vx = \noisyx) d\vx  } + \textcolor{purple}{2(1 - \beta) \int f(\noisyx)^\top \vx \, q(\noisyx) p(\vx \mid \noisyx) d\vx d\noisyx} , \\
        & = \textcolor{blue}{ 2\beta \int f(\noisyx)^\top \noisyx \, p(\noisyx) d\noisyx  } + \textcolor{purple}{2(1 - \beta) \int f(\noisyx)^\top g_\star(\noisyx) \, q(\noisyx) d\noisyx} , \\
        \label{eq:proof_right_term}
        & = \textcolor{blue}{2 \beta \int f(\noisyx)^\top \noisyx \, p(\noisyx) d\noisyx  } + \textcolor{purple}{2 (1 - \beta) \int f(\noisyx)^\top g_\star(\noisyx) \, q(\noisyx) d\noisyx},
    \end{align}
    where $g_\star(\noisyx) \triangleq \mathbb{E}_{q(\vx \mid \noisyx )}[\vx]$.

    Plugging~\eqref{eq:proof_left_term}~and~\eqref{eq:proof_right_term} into~\eqref{eq:proof_intermediate_eq}, we get
    \begin{equation}
      \begin{split}
        \label{eq:thm_final_simplification}
        \mathcal{L}(f) = & \textcolor{red} {\beta \int f(\noisyx)^\top f(\noisyx) p(\noisyx) d\noisyx } + \textcolor{orange} { (1 - \beta) \int f(\noisyx)^\top f(\noisyx) q(\noisyx) d\noisyx} \\
        & - \textcolor{blue}{ 2 \beta \int f(\noisyx)^\top \noisyx \, p(\noisyx) d\noisyx  } - \textcolor{purple}{2 (1 - \beta) \int f(\noisyx)^\top g_\star(\noisyx) \, q(\noisyx) d\noisyx}.
      \end{split}
    \end{equation}
    Using the Euler Lagrange equation~\citep{liberzon2011calculus}, we solve for $f_\star$ by taking a variational gradient of ~\eqref{eq:thm_final_simplification} with respect to $f$ and setting it equal to 0.
    \begin{gather}
        \frac{\partial \mathcal{L}(f)}{\partial f} = 0 \\
        2 \beta f(\noisyx) p(\noisyx) + 2 (1 - \beta) f(\noisyx) q(\noisyx) - 2 \beta \noisyx p(\noisyx) - 2 (1 - \beta) g_\star(\noisyx) \, q(\noisyx) = 0 \\
        2 \beta f(\noisyx) p(\noisyx) + 2 (1 - \beta) f(\noisyx) q(\noisyx) = 2 \beta \noisyx p(\noisyx) + 2 (1 - \beta) g_\star(\noisyx) \, q(\noisyx) \\
        \beta f(\noisyx) p(\noisyx) + (1 - \beta) f(\noisyx) q(\noisyx) = \beta \noisyx p(\noisyx) + (1 - \beta) g_\star(\noisyx) \, q(\noisyx) \\
        f(\noisyx) p(\noisyx) + \frac{1 - \beta}{\beta} f(\noisyx) q(\noisyx) = \noisyx p(\noisyx) + \frac{1 - \beta}{\beta} g_\star(\noisyx) \, q(\noisyx) \\
        f(\noisyx) \left[ p(\noisyx) + \frac{1 - \beta}{\beta} q(\noisyx) \right] = \noisyx p(\noisyx) + \frac{1 - \beta}{\beta} g_\star(\noisyx) \, q(\noisyx) \\
        f_\star(\noisyx) = \frac{p(\noisyx)}{p(\noisyx) + \frac{1 - \beta}{\beta} q(\noisyx)} \noisyx + \frac{\frac{1 - \beta}{\beta} q(\noisyx)}{p(\noisyx) + \frac{1 - \beta}{\beta} q(\noisyx)} g_\star(\noisyx)
    \end{gather}
    Letting $\gamma \triangleq \frac{1 - \beta}{\beta}$, $\omega(\noisyx) \triangleq \frac{p(\noisyx)}{p(\noisyx) + \gamma q(\noisyx)}$ and plugging in the definition of $g_\star(\noisyx)$, we get
    \begin{equation}
      f_\star(\noisyx) = \omega(\noisyx) \noisyx + \left( 1 - \omega(\noisyx) \right) \mathbb{E}_{q(\vx \mid \noisyx )}[\vx],
    \end{equation}
    completing the proof.
\end{proof}

\section{Training procedure for Denoise2Impute-T}
\label{app:training_details}
We train Denoise2Impute-T in a two-stage fashion.
First, we train Denoise2Impute, $g_{\bm{\theta}}$, by minimizing the cross-entropy loss, $\mathcal{L}$,
\begin{gather}
  \ell_{CE}(g_{\bm{\theta}}(\noisyx), \vx) = \sum_{d=1}^T \ell_{ce} (g_{\bm{\theta}}(\noisyx)_d, x_d), \\
  \ell_{ce}(g_{\bm{\theta}}(\noisyx)_d, x_d) = x_d \log g_{\bm{\theta}}(\noisyx)_d + (1 - x_d) \log (1 - g_{\bm{\theta}}(\noisyx)_d)),
\end{gather}
where $x_d$ is the $d$-th dimension of $\vx$ and $g_{\bm{\theta}}(\noisyx)_d$ is the $d$-th dimension of $g_{\bm{\theta}}(\noisyx)$.
To place more emphasis on instances where a 0 must be changed to a 1, we modify the cross-entropy loss to be
\begin{equation}\label{eq:scaled_nll}
  \ell_{CE} (g_{\bm{\theta}}(\noisyx), \vx) = \sum_{d=1}^T \indicator_{[\tilde{x}_d = x_d]} \ell_{ce} (g_{\bm{\theta}}(\noisyx)_d, x_d) + \lambda \indicator_{[\tilde{x}_d \neq x_d]} \ell_{ce} (g_{\bm{\theta}}(\noisyx)_d, x_d),
\end{equation}
where $\indicator_{[\tilde{x}_d \neq x_d]}$ corresponds to the case where where $\tilde{x}_d = 0$ and $x_d=1$.
In all experiments, we use $\lambda=2.0$.
During training, we also randomly mask out dimensions of $\noisyx$ with probability $p=0.3$ as we found it helped with performance.

Fixing $g_{\bm{\theta}}$, we construct Denoise2Impute-T by introducing a weighting function, $\omega_{\bm{\phi}}$
\begin{equation}
  f (\noisyx) = \omega_{\bm{\phi}}(\noisyx) x + (1 - \omega_{\bm{\phi}}(\noisyx)) g_{\bm{\theta}}(\noisyx)
\end{equation}
and optimize~\eqref{eq:scaled_nll} for $f$, i.e.,
\begin{equation}
  \ell_{CE}(f_{\bm{\theta}}(\noisyx), \vx) = \sum_{d=1}^T \indicator_{[\tilde{x}_d = x_d]} \ell_{ce} (f_{\bm{\theta}}(\noisyx)_d, x_d) + \lambda \indicator_{[\tilde{x}_d \neq x_d]} \ell_{ce} (f_{\bm{\theta}}(\noisyx)_d, x_d),
\end{equation}
where only $\bm{\phi}$ is updated.
As $\omega_{\bm{\phi}}$ is not differentiable due to the indicator function, we approximate it during training via the following continuous relaxation
\begin{equation}\label{approx:}
  \omega_{\bm{\phi}}(\noisyx) \approx \sigma(\alpha * (g_{\bm{\theta}}(\noisyx) - \phi)) ,
\end{equation}
where $\sigma(\cdot)$ denotes the sigmoid function and $\alpha$ is a temperature parameter to modulate this approximation. 
In all experiments, we set $\alpha=100$. 
Let $m_{j}^{00}$ be the average of $p_{ij}^{00}$ and $m_{j}^{01}$ be the average of $p_{ij}^{01}$ on the training data, for $j=1,\cdots, T$. 
In principle, we want $m_{j}^{00}$ to be close to zero and $m_{j}^{01}$ to be close to one on the testing data. 
When optimizing $\bm{\phi}$, we constrain the learned parameters to satisfy $0\le \phi_j\le m_{j}^{00}.$

\section{Hyperparameters and Architecture Details}
\subsection{Details of Datasets}\label{app:datasets}
For training the denoising model, we begin with two sources of EHR data with overlapping patient populations from a large healthcare institute\footnote{To comply with the double-blind submission policy, we withhold the name of the institution providing the datasets; should the paper be accepted, these details will be provided.}, $\mathcal{D}_1 = \{\tilde{\mathbf{x}}_1, \ldots, \tilde{\mathbf{x}}_{N}\}$ and $\mathcal{D}_2 = \{\tilde{\mathbf{y}}_1, \ldots, \tilde{\mathbf{y}}_M\}$. 
Each data point is a vector of binary indicators for the presence of ICD-10 diagnosis codes in a patient's EHR, where we set the number of codes to be 993, i.e., $T=993$.
As $\mathcal{D}_1$ and $\mathcal{D}_2$ have overlapping patient populations, we construct the dataset as follows: for patient $i$ in the common population, we treat the corresponding code vector from $\mathcal{D}_1$ as the noisy data point, $\tilde{\mathbf{x}}_i$.
To construct the target code vector, $\vx_i$, we perform an element-wise OR operation between $\tilde{\mathbf{x}}_i$ and the same patient's code vector from $\mathcal{D}_2$, $\tilde{\mathbf{y}}_i$; the OR operation merges any diagnoses that may be unshared between the two data sources.
Using the above procedure, a dataset $\mathcal{D}$ of size 445,345 patients is constructed where 356,276 patients are used for training and 89,069 are used for testing.

\subsection{k-NN Details}
Let $\mathcal{B}$ be a subset of patients from the training set of the clean dataset, $\mathcal{D}$, that will be used as a buffer.
In our implementation, we set $k=5$ in all experiments and specify a tolerance $\tau$.
Given an EHR vector, $\noisyx$, we use the Hamming distance to find the closest neighbor in $\mathcal{B}$, which we denote by $\vx_n$.
If the Hamming distance between $\noisyx$ and $\vx_n$ is less than $\tau$ then all the dimensions in $\noisyx$ that have a 0 are replaced by the corresponding value in $\vx_n$.
If the Hamming distance between $\noisyx$ and $\vx_n$ is greater than $\tau$, then we keep $\noisyx$ as is.

\subsection{MLP Architecture Details}
For the MLP, we used 4 hidden layers, each of width 512, and ReLU activations where a skip connection is added between consecutive layers.
The output of the network is then passed through a sigmoid function to ensure they are valid probabilities.

For training, we used the AdamW optimizer~\citep{loshchilov2017decoupled} with learning rate $3.0\times 10^{-4}$ and a weight decay of $1.0\times 10^{-5}$ with a batch size of $N=128$.

\subsection{DAE Architecture Details}
For the encoder model in denoising auto encoder (DAE), we used 4 fully connected layers, each of width 512, and ReLU activations where a ResNet module is added in the second layer, third layer and fourth layer. 
The latent dimension of the encoder is 512. 
For the decoder model of DAE, we used a symmetric architecture as the encoder. 
The output of the network is then passed through a sigmoid function to ensure they are valid probabilities. We follow the same training method as that for MLP to train the DAE.

\subsection{CDAE Architecture Details}
We apply the collaborative denoising auto-encoder \textbf{(CDAE)} from~\citet{wu2016collaborative} to our noisy data as one baseline.
For the encoder model in CDAE, we used 4 fully connected layers, each of width 512, and Tanh activations where an ResNet module is added in the second layer, third layer and fourth layer. The latent dimension of the encoder is 512. The embedding dimension of each patient is 512. For the decoder model of CDAE, we used a symmetric architecture as the encoder. The output of the network is then passed through a sigmoid function to ensure they are valid probabilities. We follow the original implementation of CDAE to set the hyperparameters in our experiment. The corruption ratio is set as 0.5. The learning rate is $1.0\times 10^{-3}$. Adam optimizer is employed to optimize CDAE. The batch size is $N=128$. The number of epochs for training is 50.

\subsection{Set Transformer Architecture Details}
The shape of input to the Set Transformer is (N, T, 1), where N is the batch size and T is the number of codes for each patient. 
We introduce $T$ learnable embeddings, one for each dimension of $\noisyx$, where each embedding is of dimension $D$.
The embeddings are appended with the input leading to a tensor of shape (N, T, D + 1).
$L$ layers of the set attention blocks (SABs) are then applied followed by
a parallel linear block that maps the output from attention blocks to a tensor of dimension (N, T, 1). 
Finally, a sigmoid function outputs the probability of each code being present. 
The neural network architecture of $g_{\bm{\theta}}$ is detailed in the Figure~\ref{fig:arch}. 

\begin{figure}[ht!]
  \begin{center}
\includegraphics[width=9cm]{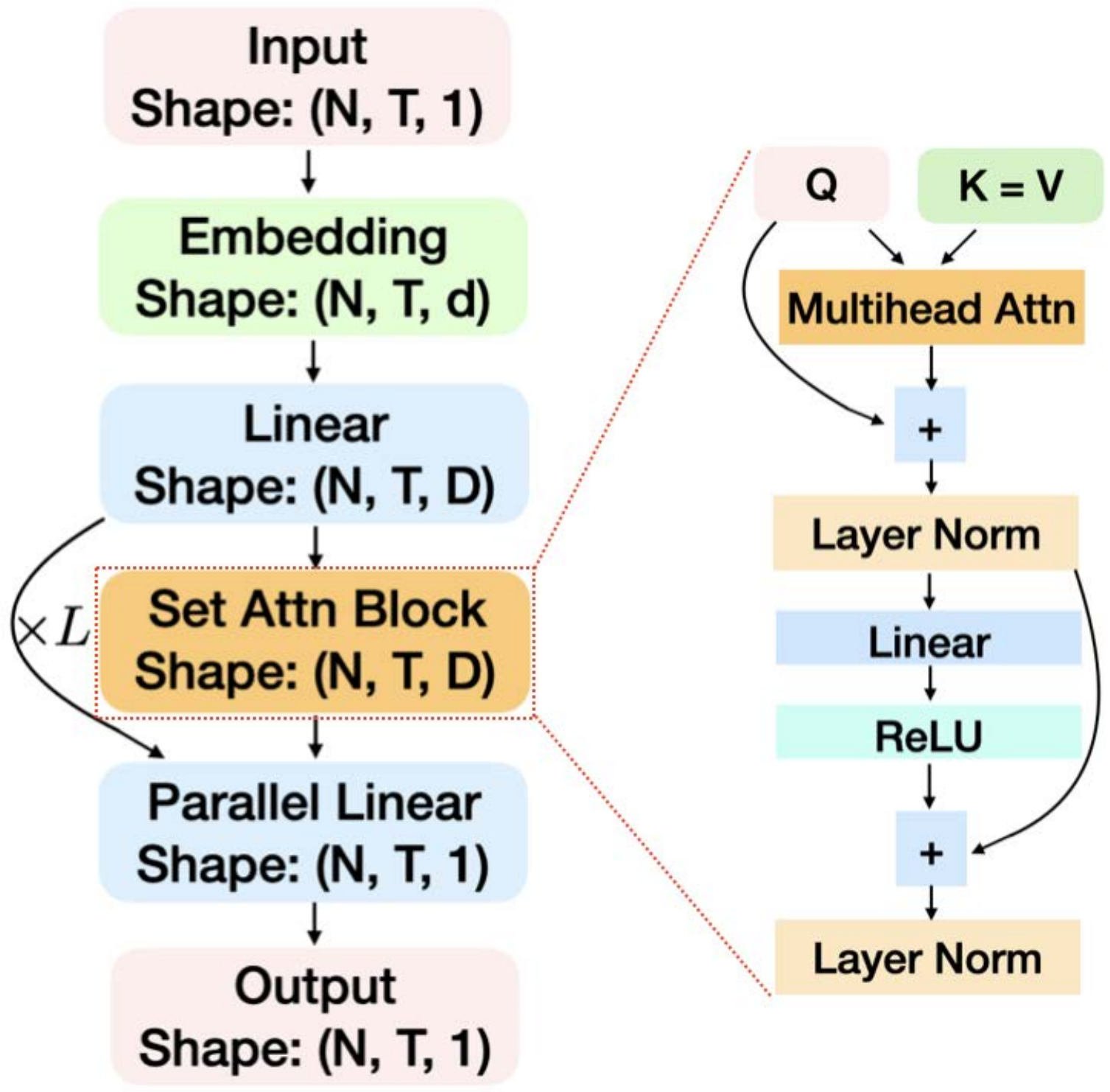}
\caption{Neural Network Architecture of the denoising model $g_{\bm{\theta}}$, where one of the key modules is the set attention block (SAB). 
The tensor shape of outputs in each layer is provided. We apply $L$ layers of the SAB. \label{fig:arch}}
 \end{center}
\end{figure}

In all our experiments, we set $D=200$ and $L=4$, where for each SAB, we set the number of heads to be 10.
We use the AdamW optimizer~\citep{loshchilov2017decoupled} with learning rate $3.0\times 10^{-4}$ and a weight decay of $1.0\times 10^{-5}$ with a batch size of $N=48$. All methods in this paper is trained with 50 epochs.
\begin{table}[ht!]
  \centering
  \resizebox{\columnwidth}{!}{\begin{tabular}{ |c | c | c | c |c | c | c | c |}
    \hline
    Methods  & $\mathcal{D}_1$ & Prevalence & MissForest & MLP & CDAE & D2Impute & D2Impute-T \\ \hline
    Micro AUPRC & 62.67 & 70.39 & 64.14 & 79.10 & 79.13 & 81.28 & {\bf 81.64}  \\ \hline
    Macro AUPRC & 51.86 & 52.37 & 52.23 & 61.91  & 62.16 & 64.89 & {\bf 65.53}  \\ \hline
  \end{tabular}}  
\captionof{table}{Micro AUPRC and Macro AUPRC (higher is better) on the test set where noisy vectors were obtained from $\mathcal{D}_1$. D2Impute is short for Denoise2Impute. \label{app:tab:test_set_results}}
\end{table}
\subsection{Gradient Boosted Tree Details} 
For the downstream tasks, we used a light gradient boosting decision tree model (LGBM)~\citep{ke2017lightgbm} as it had robust downstream prediction performance.
In all experiments, we set the hyperparameters of LGBM as follows: $\mathrm{n\_estimators}=1000$, $\mathrm{learning\_rate}=0.05$ $\mathrm{max\_depth}=10$, $\mathrm{reg\_alpha}=0.5$, $\mathrm{reg\_lambda}=0.5$, $\mathrm{scale\_pos\_weight}=1$, $\mathrm{min\_data\_in\_bin}=128.$

\section{Supplementary Experiments}\label{app:supp_experiments}
In the following section, we provide additional experimental results where we also investigate how the performance changes when we treat $\mathcal{D}_2$ as the noisy input instead of $\mathcal{D}_1$.
\subsection{Details for ICD code prediction task\label{app:icd_details}}
Given the training set $\mathcal{D}$, which consists of 356,276 patients, we used 222,672 patients to train Denoise2Impute and Denoise2Impute-T and the other 133,604 patients are used to train the LGBM model.
We selected ICD codes corresponding to seven common chronic conditions: Type-II Diabetes (T2D), Atrial Fibrillation (AFib), Asthma, Chronic Obstructive Pulmonary Disease (COPD), Chronic Kidney Disease (CKD), Hypertensive Heart Disease (HTN-Heart) and Osteoarthritis.
We also randomly selected eight additional ICD codes for testing, where we ensured that their prevalence was greater than 6\%.
The description for the eight randomly selected ICD codes are shown below~(Table~\ref{app:tab:eight_disease_captions}).
Here we present detailed results on both, the seven common chronic diseases~(Table~\ref{app:tab:d1_common_chronic_codes}) and the eight randomly chosen codes~(Table~\ref{app:tab:d1_eight_random_codes}), which correspond to Figure~\ref{fig:icd_prediction}A and Figure~\ref{fig:icd_prediction}B, respectively.
\begin{table}[ht!] 
  \centering
    \begin{tabular}{ |c | c |}
            \hline
            ICD-10 codes & Description of Diagnosis ICD-10 codes  \\ \hline
            Z37&Outcome of delivery \\ \hline
            Z00&Encounter for general examination without CSRD \\ \hline
            Z12&Encounter for screening for malignant neoplasms \\ \hline
            M54&Dorsalgia \\ \hline
            Z23&Encounter for immunization \\ \hline
            Z01&Encounter for other special examination without CSRD  \\ \hline
            V22&Motorcycle rider injured in collision with two- or three-wheeled motor vehicle \\ \hline
            Z34&Encounter for supervision of normal pregnancy \\ \hline     
      \end{tabular}
  \caption{Description of additional ICD codes from Fig.~\ref{fig:icd_prediction}. CSRD is short for without complaint, suspected or reported diagnosis.}
  \label{app:tab:eight_disease_captions}
\end{table}

\begin{table}[ht!]
  \centering
  \begin{tabular}{ |c | c | c | c | c | c | c | c |}
  \hline
      &  T2D & AFib &  Asthma & COPD & CKD & HTN & Osteoarthritis \\ \hline
    Ground Truth & 90.77 & 58.40 & 73.63 & 71.59 & 80.03 & 44.77 & 48.76 \\ 
    & $\pm$ 0.14 & $\pm$ 0.50 & $\pm$ 0.29 & $\pm$ 0.46 & $\pm$ 0.44 & $\pm$ 0.64 & $\pm$ 0.32  \\ \hline
    $\mathcal{D}_1$ & 78.47 &  40.85 & 47.26  &  58.08  &  61.90  &  32.29 & 38.13 \\ 
    & $\pm$0.14 & $\pm$ 0.50 & $\pm$0.29 & $\pm$ 0.46 &  $\pm$ 0.47 & $\pm$0.58 &   $\pm$0.31 \\ \hline
    softImpute & 77.28 & 40.06 & 46.43 & 58.15 & 61.78 & 32.46 & 37.89 \\ 
    & $\pm$ 0.16 & $\pm$ 0.48 & $\pm$ 0.31 & $\pm$ 0.48 & $\pm$ 0.51 & $\pm$ 0.63 & $\pm$ 0.36  \\ \hline
    k-NN & 78.36 & 41.09 & 47.42 & 58.31 & 61.94 & 32.23 & 38.23 \\
    & $\pm$ 0.15 & $\pm$ 0.53 & $\pm$ 0.17 & $\pm$ 0.24 & $\pm$ 0.44 & $\pm$ 0.64 & $\pm$ 0.25 \\ \hline
    MLP & 78.71 & 41.32 & 47.69 & 58.55 & 62.28 & 32.02 & 38.59 \\  
    & $\pm$ 0.13 & $\pm$ 0.60 & $\pm$ 0.22 & $\pm$ 0.38 & $\pm$ 0.27 & $\pm$ 0.54 & $\pm$ 0.24  \\ \hline
    MLP-T & 79.04 & 41.59 & 47.93 & 58.79 & 62.75 & 32.23 & 38.87 \\  
    & $\pm$ 0.16 & $\pm$ 0.57 & $\pm$ 0.26 & $\pm$ 0.37 & $\pm$ 0.32 & $\pm$ 0.59 & $\pm$ 0.27  \\ \hline
    DAE & 78.66 &  41.38 & 47.67 &  58.71 & 62.41 & 32.07 & 38.67 \\
      & $\pm$ 0.12 & $\pm$ 0.55 & $\pm$ 0.24 & $\pm$ 0.38 & $\pm$ 0.33 & $\pm$ 0.61 & $\pm$ 0.29 \\ \hline
    CDAE & 78.74 & 41.45 & 47.64 &  58.83 &  62.51 & 32.11 & 38.62 \\
    &$\pm$ 0.14 & $\pm$ 0.53 & $\pm$ 0.27 & $\pm$0.38 & $\pm$ 0.33 & $\pm$ 0.42 & $\pm$ 0.30   \\ \hline
    D2Impute  & 79.35 & 42.11 &  47.81  & 59.55 &  63.05 & 32.39  &   39.05 \\  
    & $\pm$ 0.14 & $\pm$ 0.52 & $\pm$ 0.2 & $\pm$ 0.29 & $\pm$ 0.42 & $\pm$ 0.49 & $\pm$ 0.30 \\ \hline
    D2Impute-T & {\bf 79.73} & {\bf 42.74} & {\bf 48.47} & {\bf 60.07} & {\bf 63.85} & {\bf 32.96} & {\bf 39.36} \\ 
    & $\pm$ 0.10 & $\pm$ 0.51 & $\pm$ 0.29 & $\pm$ 0.38 & $\pm$ 0.42 & $\pm$ 0.52 & $\pm$ 0.26 \\ \hline
\end{tabular}
\caption{Results for 7 ICD codes corresponding to common chronic conditions. Mean AUPRC and $95\%$ confidence interval (CI) are provided. We bootstrap 80\% of the test data and repeat 50 times to compute CI. "T2D": type-II diabetes. "AFib": atrial fibrillation. "COPD": chronic obstructive pulmonary disease. "CKD": chronic kidney disease. "HTN": hypertensive heart disease. D2Impute is short for Denoise2Impute.}
\label{app:tab:d1_common_chronic_codes}
\end{table}

\begin{table}[ht!]
  \centering
    \begin{tabular}{ |c | c | c | c | c | c | c |c | c |}
    \hline
         &  Z37 & Z00 & Z12 & M54 & Z23 & Z01 & V22 & Z34 \\ \hline
         Ground Truth & 93.86 &  90.29 &  89.15 &  82.5 & 78.06 &  79.38 &  94.81 &  93.62 \\ 
    & $\pm$ 0.15 & $\pm$ 0.07  & $\pm$ 0.06 &  $\pm$ 0.12 & $\pm$ 0.09 &  $\pm$ 0.13 & $\pm$ 0.09 & $\pm$ 0.16 \\  \hline
         $\mathcal{D}_1$ & 80.95 & 81.70 & 79.81 & 82.00 & 72.08 & 72.28 & 75.16 & 86.17 \\  
         & $\pm$ 0.31 & $\pm$ 0.09 & $\pm$ 0.10 & $\pm$ 0.09 & $\pm$ 0.1 & $\pm$ 0.09 & $\pm$ 0.31 & $\pm$ 0.21 \\ \hline
         softImpute & 80.41 & 80.56 & 79.32 & 81.35 & 71.87 & 72.41 & 75.69 & 86.58 \\  
         & $\pm$ 0.33 & $\pm$ 0.10 & $\pm$ 0.11 & $\pm$ 0.10 & $\pm$ 0.14 & $\pm$ 0.11 & $\pm$ 0.32 & $\pm$ 0.26 \\ \hline
         k-NN & 80.84 & 81.53 & 79.95 & 82.06 & 72.22 & 72.43 & 75.84 & 86.46 \\  
         & $\pm$ 0.32 & $\pm$ 0.09 & $\pm$ 0.09 & $\pm$ 0.10 & $\pm$ 0.10 & $\pm$ 0.14 & $\pm$ 0.29 & $\pm$ 0.20 \\ \hline
      MLP & 81.41 & 81.71 & 80.86 & 82.44 & 72.86 & 72.82 & 76.27 & 87.0 \\  
      & $\pm$ 0.28 & $\pm$ 0.08 & $\pm$ 0.07 & $\pm$ 0.08 & $\pm$ 0.07 & $\pm$ 0.14 & $\pm$ 0.30 & $\pm$ 0.21 \\ \hline
      MLP-T & 81.67 & 82.06 & 81.02 & 82.51 & 72.98 & 73.05 & 76.83 & 87.54 \\  
      & $\pm$ 0.31 & $\pm$ 0.11 & $\pm$ 0.09 & $\pm$ 0.07 & $\pm$ 0.11 & $\pm$ 0.16 & $\pm$ 0.35 & $\pm$ 0.24 \\ \hline
      DAE &  81.36 & 81.83 & 80.79 & 82.25  &  72.79 & 72.73 & 76.35 &  87.12 \\
          &  $\pm$ 0.31 & $\pm$ 0.08 & $\pm$ 0.09 & $\pm$ 0.08 &  $\pm$ 0.08 & $\pm$ 0.14 & $\pm$ 0.32 & $\pm$ 0.25 \\ \hline
          CDAE & 81.32 & 81.89 & 80.81 & 82.31 & 72.75 & 72.86 & 76.39 &  87.08 \\
        & $\pm$ 0.32 & $\pm$ 0.09 & $\pm$ 0.10 & $\pm$ 0.09 & $\pm$ 0.11 & $\pm$ 0.12 & $\pm$ 0.34 & $\pm$ 0.24 \\ \hline
      D2Impute & 81.42 & 82.54 &  81.27 &  82.4  & 73.17 & 73.37  &  77.36  & 87.89 \\  
      & $\pm$ 0.30 & $\pm$ 0.07 & $\pm$ 0.08 & $\pm$ 0.09 & $\pm$ 0.14 & $\pm$ 0.10 & $\pm$ 0.28 & $\pm$ 0.23 \\ \hline
      D2Impute-T &   {\bf 82.46}  &  {\bf 82.75}  &  {\bf 81.56}  &  {\bf 82.75}  &  {\bf 73.54}  &  {\bf 73.52}  &  {\bf 78.48}  &  {\bf 88.47} \\  
      & $\pm$ 0.30 & $\pm$ 0.07 & $\pm$ 0.08 & $\pm$ 0.09 & $\pm$ 0.07 & $\pm$ 0.09 & $\pm$ 0.38 & $\pm$ 0.18 \\ \hline
    \end{tabular}
    \caption{Results for eight randomly chose ICD code predictions, where each ICD code has a prevalence larger than 6\%. 
    Mean AUPRC and $95\%$ confidence intervals (CI) are provided. We bootstrap 80\% of the test data and repeat 50 times to compute CI. D2Impute is short for Denoise2Impute.}
    \label{app:tab:d1_eight_random_codes}
\end{table}

\newpage
\subsection{Benchmarking ICD Code Prediction Tasks using $\mathcal{D}_2$}
To investigate how the performance on the ICD code prediction task varies as a function of which noisy dataset we use, we present results for imputation methods where $\mathcal{D}_2$ is used as the noisy dataset instead of $\mathcal{D}_1$.
The corresponding AUPRCs are reported in Tables~\ref{app:tab:d2_common_chronic_codes}~\&~\ref{app:tab:d2_eight_random_codes}. 
Compared with Tables~\ref{app:tab:d1_common_chronic_codes}~\&~\ref{app:tab:d1_eight_random_codes}, we observe that the absolute value of AUPRC increased across most ICD-10 codes and methods suggesting that $\mathcal{D}_2$ contains more information that is predictive of $\mathcal{D}$, as compared to $\mathcal{D}_1$. 
\begin{table}[ht!]
  \centering
  \begin{tabular}{ |c | c | c | c | c | c | c | c |}
    \hline
         &  T2D & AFib &  Asthma & COPD & CKD & HTN & Osteoarthritis \\ \hline
         Ground Truth & 90.77 & 58.40 & 73.63 & 71.59 & 80.03 & 44.77 & 48.76 \\ 
        & $\pm$ 0.14 & $\pm$ 0.50 & $\pm$ 0.29 & $\pm$ 0.46 & $\pm$ 0.42 & $\pm$ 0.54 & $\pm$ 0.32  \\ \hline
        $\mathcal{D}_2$ & 88.71 & 55.38 & 71.77 & 67.24 & 76.02 & 41.87 & 45.07  \\  
        & $\pm$ 0.07 & $\pm$ 0.32 & $\pm$ 0.18 & $\pm$ 0.38 & $\pm$ 0.27 & $\pm$ 0.68 & $\pm$ 0.23 \\ \hline
        softImpute & 86.98 & 54.62 & 70.06 & 66.13 & 74.06 & 40.14 & 43.29  \\  
        & $\pm$ 0.14 & $\pm$ 0.39 & $\pm$ 0.21 & $\pm$ 0.41 & $\pm$ 0.29 & $\pm$ 0.41 & $\pm$ 0.24 \\ \hline
        k-NN & 88.16 & 54.95 & 71.21 & 66.88 & 75.17 & 40.83 & 44.26  \\  
        & $\pm$ 0.11 & $\pm$ 0.38 & $\pm$ 0.19 & $\pm$ 0.36 & $\pm$ 0.33 & $\pm$ 0.67 & $\pm$ 0.23 \\ \hline
        MLP & 88.82 & 55.88& 71.69 & 67.82& 76.11& 43.49 & 45.78  \\  
        & $\pm$ 0.10 & $\pm$ 0.39 & $\pm$ 0.15 & $\pm$ 0.34 & $\pm$ 0.28 & $\pm$ 0.47 & $\pm$ 0.32 \\ \hline
        MLP-T & 89.04 & 56.31 & 71.95 & 68.14 & 76.53 & 43.72 & 46.11  \\  
        & $\pm$ 0.15 & $\pm$ 0.41 & $\pm$ 0.19 & $\pm$ 0.37 & $\pm$ 0.32 & $\pm$ 0.51 & $\pm$ 0.34 \\ \hline
        DAE & 88.79 & 55.97 & 71.76 & 68.06 & 76.24 & 43.38 & 45.92  \\  
        & $\pm$ 0.12 & $\pm$ 0.41 & $\pm$ 0.13 & $\pm$ 0.36 & $\pm$ 0.29 & $\pm$ 0.45 & $\pm$ 0.35 \\ \hline
        CDAE & 88.87 & 56.01 & 71.72 & 68.24 & 76.35 & 43.41 & 46.03  \\  
        & $\pm$ 0.14 & $\pm$ 0.45 & $\pm$ 0.14 & $\pm$ 0.33 & $\pm$ 0.31 & $\pm$ 0.52 & $\pm$ 0.35 \\ \hline
        D2Impute & {\bf 89.36} & 57.32 & 72.37 &  69.75 & {\bf 78.11} & {\bf 44.51} & {\bf 47.31} \\  
        & $\pm$ 0.13 & $\pm$ 0.36 & $\pm$ 0.15 & $\pm$ 0.22 & $\pm$ 0.25 & $\pm$ 0.53 & $\pm$ 0.33 \\ \hline
        D2Impute-T &  89.22 & {\bf 57.49} &  {\bf 72.45} &  {\bf 69.99} &  77.95 & 44.50 &  47.07 \\ 
        & $\pm$ 0.12 & $\pm$ 0.43 & $\pm$ 0.18 & $\pm$ 0.35 & $\pm$ 0.34 & $\pm$ 0.54 & $\pm$ 0.36 \\ \hline
    \end{tabular}
    \caption{Results for 7 ICD codes corresponding to common chronic conditions. Mean AUPRC and $95\%$ confidence intervals (CI) are provided. We bootstrap 80\% of the test data and repeat 50 times to compute CI. "T2D": type-II diabetes. "AFib": atrial fibrillation. "COPD": chronic obstructive pulmonary disease. "CKD": chronic kidney disease. "HTN": hypertensive heart disease. D2Impute is short for Denoise2Impute.}
\label{app:tab:d2_common_chronic_codes}
\end{table}

\begin{table}[ht!] 
  \centering
  \begin{tabular}{|c | c | c | c | c | c | c | c | c |}
          \hline
      &  Z37 & Z00 & Z12 & M54 & Z23 & Z01 & V22 & Z34 \\ \hline
      Ground Truth & 93.86 & 90.29 & 89.15 & 82.5 & 78.06 & 79.38 & 94.81 & 93.62 \\ 
    & $\pm$ 0.15 & $\pm$ 0.07  & $\pm$ 0.06 & $\pm$ 0.12 & $\pm$ 0.09 & $\pm$ 0.13 & $\pm$ 0.09 & $\pm$ 0.16 \\  \hline
    $\mathcal{D}_2$ & 91.98& 89.36& 87.11& 78.28& 75.92& 75.20 & 93.95 & 90.91  \\  
    & $\pm$ 0.14 & $\pm$ 0.06 & $\pm$ 0.07 & $\pm$ 0.09 & $\pm$ 0.10 & $\pm$ 0.11 & $\pm$ 0.08 & $\pm$ 0.15 \\ \hline
    softImpute & 89.74 & 87.25 & 86.33 & 76.44 & 73.18 & 74.46 & 90.16 & 88.17  \\  
    & $\pm$ 0.21 & $\pm$ 0.09 & $\pm$ 0.10 & $\pm$ 0.11 & $\pm$ 0.14 & $\pm$ 0.16 & $\pm$ 0.11 & $\pm$ 0.18 \\ \hline
    k-NN & 90.22 & 88.84 & 86.78 & 77.83 & 76.14 & 74.62 & 92.75 & 90.68  \\  
    & $\pm$ 0.16 & $\pm$ 0.07 & $\pm$ 0.08 & $\pm$ 0.10 & $\pm$ 0.13 & $\pm$ 0.12 & $\pm$ 0.09 & $\pm$ 0.16 \\ \hline
    MLP & 92.43 & 89.44 & 87.44 & 78.54 & 76.26 & 75.47 & 93.99& 91.10  \\  
    & $\pm$ 0.17 & $\pm$ 0.06 & $\pm$ 0.07 & $\pm$ 0.13 & $\pm$ 0.08 & $\pm$ 0.11 & $\pm$ 0.12 & $\pm$ 0.21 \\ \hline
    MLP-T & 92.78 & 89.49 & 87.76 & 78.95 & 76.41 & 75.84 & 94.11 & 91.75  \\  
    & $\pm$ 0.18 & $\pm$ 0.08 & $\pm$ 0.11 & $\pm$ 0.12 & $\pm$ 0.10 & $\pm$ 0.14 & $\pm$ 0.13 & $\pm$ 0.19 \\ \hline
    DAE & 92.54 & 89.32 & 87.49 & 78.65 & 76.21 & 75.56 & 93.87 & 91.34  \\  
    & $\pm$ 0.19 & $\pm$ 0.09 & $\pm$ 0.06 & $\pm$ 0.18 & $\pm$ 0.11 & $\pm$ 0.15 & $\pm$ 0.14 & $\pm$ 0.19 \\ \hline
    CDAE & 92.61 & 89.37 & 87.56 & 78.79 & 76.28 & 75.74 & 93.96 & 91.55  \\  
    & $\pm$ 0.16 & $\pm$ 0.07 & $\pm$ 0.09 & $\pm$ 0.16 & $\pm$ 0.14 & $\pm$ 0.17 & $\pm$ 0.13 & $\pm$ 0.21 \\ \hline
    D2Impute &  93.41 &  89.60 &  87.88 & 79.79 &  76.34 & 76.42 & {\bf 94.48} & {\bf 92.78} \\  
    & $\pm$ 0.17 & $\pm$ 0.05 & $\pm$ 0.06 & $\pm$ 0.11 & $\pm$ 0.10 & $\pm$ 0.12 & $\pm$ 0.13 & $\pm$ 0.14 \\ \hline
    D2Impute-T & {\bf 93.44} & {\bf 89.64} &  {\bf 87.93} &  {\bf 79.86}&  {\bf 76.61}&  {\bf 76.50}&  94.40 & 92.77  \\
    & $\pm$ 0.12 & $\pm$ 0.06  & $\pm$ 0.08 & $\pm$ 0.12 & $\pm$ 0.07 & $\pm$ 0.14 & $\pm$ 0.12 & $\pm$ 0.15 \\  \hline
    \end{tabular}
    \caption{Results for eight randomly chose ICD code predictions, where each ICD code has a prevalence larger than 6\%. 
    Mean AUPRC and $95\%$ confidence intervals (CI) are provided. We bootstrap 80\% of the test data and repeat 50 times to compute CI. D2Impute is short for Denoise2Impute.}
    \label{app:tab:d2_eight_random_codes}
  \end{table}

\newpage
\subsection{Benchmarking Denoising on Hospital Readmission Prediction}
\label{app:hospital}
We also benchmark the denoising models on a popular healthcare prediction task from the literature, predicting hospital readmission from ICD-10 codes~\citep{kansagara2011risk, huang2019clinicalbert, mahmoudi2020use, jiang2023health}. 
While ICD-10 codes can be used to predict hospital readmission, the fact that readmission prediction is not an ICD-10 code itself enhances its credibility as a benchmarking dataset since we no longer need to set a dimension to all 0s before performing denoising.
The size of the training dataset is 1,124,550, which is used to train a LGBM model, and the size of the testing dataset is 545,797, which is used to evaluate the performance the LGBM model. 
We observe that our proposed methods, Denoise2Impute and Denoise2Impute-T, statistically significantly outperform existing imputation methods (Table~\ref{tab4.2}).
\begin{table}[ht!]
  \centering
  \begin{tabular}{ |c |c | c| c | c | c | c | c |}
   \hline
    Methods & $\mathcal{D}_1$ & softImpute & MLP & DAE & CDAE & D2Impute & D2Impute-T \\  \hline
    Mean &  20.74 &  18.58 &  20.81  & 20.87 & 20.89 & 21.47 & {\bf 21.51} \\
    CI & $\pm$ 0.15 & $\pm$ 0.16 &$\pm$ 0.13 & $\pm$ 0.12 & $\pm$ 0.14 & $\pm$ 0.10 & $\pm$ 0.10 \\ \hline
    \end{tabular}
    \caption{Hospital readmission prediction on $\mathcal{D}_1$ dataset. Mean AUPRC and $95\%$ confidence intervals (CI). We bootstrap 80\% of the test data and repeat 50 times to compute CI. D2Impute is short for Denoise2Impute.}
    \label{tab4.2}
  \end{table}

  \newpage
  \subsection{Benchmarking ICD Code Predictions using Synthetic $\mathcal{D}_1$}\label{app:synthetic}
  To investigate how the noise distribution affects the performance of the models, we created an artificial noisy dataset, Noisy $\mathcal{D}_1$, from $\mathcal{D}$, where the noise was distributed such that the prevalence of Noisy $\mathcal{D}_1$ matches that of $\mathcal{D}_1$. Specifically, a binary mask was randomly sampled per patient and per ICD code and was multiplied with $\mathcal{D}$.
  Results are demonstrated in Figure~{\ref{fig:synthetic_example}}.
  \begin{figure}[ht!]
    \begin{center}
      \includegraphics[width=\textwidth]{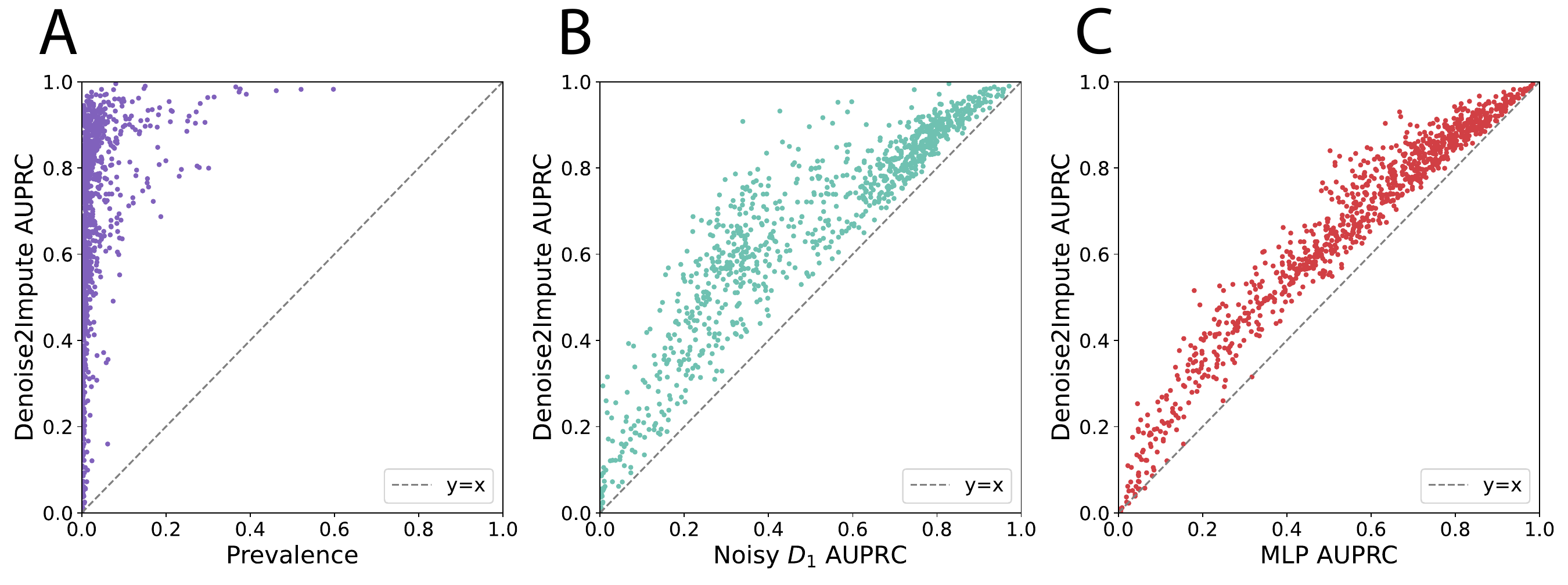}
      \caption{Results on ICD code prediction where the denoising models are trained using Noisy $\mathcal{D}_1$.
      We compare the per ICD code AUPRC of ImputEHR with those of A) imputing 0's with prevalence B) using  Noisy $\mathcal{D}_1$ and C) denoising using an MLP.
      }
      \label{fig:synthetic_example}
     \end{center}
  \end{figure}

\end{document}